\documentclass{article}

     \PassOptionsToPackage{numbers, compress}{natbib}

    \usepackage[preprint]{neurips_2018}

\title{Don't Forget Your Teacher:\\ A Corrective Reinforcement Learning Framework}

\author{%
 Mohammadreza~Nazari \\
 \texttt{mrza.nazari@gmail.com} \\
  \And
    Majid~Jahani \\
 \texttt{maj316@lehigh.edu} \\
   \AND
    Lawrence~V.~Snyder \\
 \texttt{lvs2@lehigh.edu} \\
    \And
    Martin Tak\'a\v{c} \\
 \texttt{takac.MT@gmail.com} 
 \AND
  	{ \normalfont Department of Industrial and Systems Engineering}\\
	Lehigh University, Bethlehem, PA 18015 
}

\usepackage{natbib}
\usepackage{graphicx}

\usepackage[utf8]{inputenc} 
\usepackage[T1]{fontenc}    
\usepackage{hyperref}       
\usepackage{url}            
\usepackage{booktabs}       
\usepackage{amsfonts}       
\usepackage{nicefrac}       
\usepackage{microtype}      

\usepackage{amsmath,amsthm,enumitem,caption,multirow,multicol,algorithmic,algorithm,array,bbm}
\usepackage{comment}
\usepackage{subcaption}
\usepackage{rotating,booktabs}
\usepackage{afterpage}
\usepackage{tabu}
\usepackage{wrapfig}
\usepackage{mathrsfs}  
\usepackage{mathtools}
\DeclareMathOperator*{\softmax}{softmax}

\makeatletter
\newcommand{\dashrule}[1][black]{%
	\color{#1}\rule[\dimexpr.5ex-.2pt]{4pt}{.4pt}\xleaders\hbox{\rule{4pt}{0pt}\rule[\dimexpr.5ex-.2pt]{4pt}{.4pt}}\hfill\kern0pt%
}
\usepackage{tikz}
\usetikzlibrary{shapes,arrows, decorations.markings}
\graphicspath{{./figures/}}

\definecolor{myGreen}{RGB}{0,175,0}

\DeclareMathOperator*{\argmin}{argmin} 
 
\newtheorem*{remark}{Remark}
\newtheorem{corollary}{Corollary}
\newtheorem{assumption}{Assumption}
\newtheorem{theorem}{Theorem}
\newtheorem{lemma}{Lemma}
\newtheorem{proposition}{Proposition}

\usepackage{enumitem}

\allowdisplaybreaks 

\begin{document}
\setlength{\abovedisplayskip}{3pt}
\setlength{\belowdisplayskip}{3pt}

\maketitle
	\begin{abstract}
		Although reinforcement learning (RL) can provide reliable solutions in many settings, practitioners are often wary of the discrepancies between the RL solution and their status quo procedures. Therefore, they may be reluctant to adapt to the novel way of executing tasks proposed by RL. On the other hand, many real-world problems require relatively small adjustments from the status quo policies to achieve improved performance. Therefore, we propose a student-teacher RL mechanism in which the RL (the ``student'') learns to maximize its reward, subject to a constraint that bounds the difference between the RL policy and the ``teacher'' policy. The teacher can be another RL policy (e.g., trained under a slightly different setting), the status quo policy, or any other exogenous policy. We formulate this problem using a stochastic optimization model and solve it using a primal-dual policy gradient algorithm. We prove that the policy is asymptotically optimal. However, a naive implementation suffers from high variance and convergence to a stochastic optimal policy. With a few practical adjustments to address these issues, our numerical experiments confirm the effectiveness of our proposed method in multiple GridWorld scenarios.\vspace{-5pt}
	\end{abstract}
	
	\section{Introduction}\vspace{-5pt}
	
We encourage using a new paradigm called {\em corrective RL}, in which a reinforcement learning (RL) agent is trained to maximize its reward while not straying ``too far'' from a previously defined policy. The motivation is twofold: (1) to provide a gentler transition for a decision-maker who is accustomed to using a certain policy but now considers implementing RL, and (2) to develop a framework for gently transitioning from one RL solution to another when the underlying environment has changed.

	RL has recently achieved considerable success in artificially created environments, such as Atari games \cite{mnih2013playing,mnih2016asynchronous} or robotic simulators \cite{lillicrap2015continuous}. Exploiting the power of neural networks in RL algorithms has been shown to exhibit super-human performance by enabling automatic feature extraction and policy representations, but  real-world applications are still very limited, conceivably due to lack of representativity of the optimized policies. Over the past few years, a major portion of the RL literature has been developed for RL agents with {\em no prior information} about how to do a task. Typically, these algorithms start with random actions and learn while interacting with the environment through trial and error. However, in many practical settings, prior information about good solutions {\em is available}, whether from a previous RL algorithm or a human decision-maker's prior experience. Our approach trains the RL agent to make use of this prior information when optimizing, in order to avoid deviating too far from a target policy.
	
	Although  toy environments and Atari games are prevalent in the RL literature due to their simplicity, RL has recently been trying to find its path to  real-world applications such as recommender systems \cite{chen2018stabilizing}, transportation \cite{nazari2018reinforcement}, Internet of Things \cite{feng2017smart}, supply chain \cite{gijsbrechts2018can,oroojlooyjadid2017deep} and various control tasks in robotics \cite{gu2017deep}. In all of these applications, there is a crucial risk that the new policy might not operate logically or safely, as one was expecting it to do. A policy that attains a large reward but deviates too much from a known policy---which follows logical steps and processes---is not desirable for these tasks.  For example, users of a system who were accustomed to the old way of doing things would likely find it hard to switch to a newly discovered policy, especially if the benefit of the new policy is not obvious or immediately forthcoming. Indeed, we argue that many real-world tasks only need a small \textit{corrective} fix to the currently running policies to achieve their desired goals, instead of designing everything from scratch. Throughout this paper, we adhere to this paradigm---we call it ``corrective RL''---which utilizes an acceptable policy as a gauge when designing novel policies.
We consider two agents, namely a {\em teacher} and a {\em student}. Our main question is how the student can improve upon the teacher's policy while not deviating too far from it. More formally, we would like to train a student in a way that it maximizes a long-term RL objective while keeping its own policy close to that of the teacher. 
	
	For example, consider an airplane that is controlled by an autopilot that follows the shortest haversine path policy towards the destination. Then, some turbulence occurs, and we want to modify the current path to avoid the turbulence. A ``pure'' RL algorithm would re-optimize the trajectory from scratch, potentially deviating too far from the optimal path in order to avoid the turbulence. Corrective RL would ensure that the adjustments to the current policy are small, so that the flight follows a similar path and has a similar estimated time of arrival, while ensuring that the passengers experience a more comfortable (less turbulent) flight. Another example is in predictive maintenance, where devices are periodically inspected for possible failures. Inspection schedules are usually prescribed by the device designers, but many environmental conditions affect failure rates, hence there is no guarantee that factory schedules are perfect. If the objective is to reduce downtime with only slight adjustments to the current schedules, conventional RL algorithms would have a hard time finding such policies. 
	
Similar concerns arise in other business and engineering domains as well, including supply chain management, queuing systems, finance, and robotics. For example, an enormous number of exact and inexact methods have been proposed for  classical inventory management problems under some assumptions on the demand, e.g., that it follows a normal distribution \cite{snyder2019fundamentals}. Once we add more stochasticity to the demand distribution or consider more complicated cost functions, these problems often become intractable using classical methods. Of course, vanilla RL can provide a mechanism for solving more complex inventory optimization problems, but practitioners may prefer policies that are simple to implement and intuitive to understand. Corrective RL can help steer the policy toward the preferred ones, while still maintaining near-optimal performance. 
Given these examples, one can interpret our approach as an improvement on black-box heuristics, which uses a data-driven approach to improve the performance of these algorithms without dramatic reformulations.

	

	The contributions of this work are as follows: \textit{i}) we introduce a new paradigm for RL tasks, convenient for many real-world tasks, that improves upon the currently running system's policy with a small perturbation of the policy; \textit{ii}) we formulate this problem using an stochastic optimization problem and propose a primal--dual policy gradient algorithm which we prove to be asymptotically optimal, and \textit{iii}) using practical adjustments, we illustrate how an RL framework can act as an improvement heuristic. We show the effectiveness and properties of the algorithm in multiple GridWorld motion planning experiments.\vspace{-5pt}

\section{Problem Definition} \vskip-5pt
	We consider the standard definition of a Markov decision process (MDP) using a tuple $(\mathcal{X},\mathcal{A},C,P,P_0)$. In our notation, $\mathcal{X}\coloneqq\mathcal{X}'\cup\{x_{term}\}=\{1,2,\ldots,n,x_{term}\}$ is the state space, where $\mathcal{X}'$ is the set of transient states and $x_{term}$ is the terminal state; $\mathcal{A}$ is the set of actions; $C: \mathcal{X}\times \mathcal{A} \to [0,C_{max}]$ is the cost function; $P$ is the transition probability distribution; and $P_0$ is the distribution of the initial state $x_0$. At each time step $t=0,1,\ldots$, the agent observes $x_t\in\mathcal{X}$, selects $a_t\in\mathcal{A}$, and incurs a cost $c_t = C(x_t,a_t)$. Selecting the action $a_t$ at state $x_t$ transitions the agent to the next state $x_{t+1}\sim P(\cdot|x_t,a_t)$.
	
	Consider two agents, a teacher and a student. The teacher has prior knowledge about the task and prescribes an action for any state that the student encounters, and the student has the authority to follow the teacher's action or act independently based on its own learned policy. Let $\pi_S$ denote the policy of the student and $\pi_T$ be the policy of the teacher, where both $\pi_S$ and $\pi_T$ are stationary stochastic policies defined as a mapping from a state--action pair to a probability distribution, i.e., $\pi_i:\mathcal{X}\times \mathcal{A} \to [0,1]$, $i\in\{S,T\}$. For example, $\pi_S(a|x)$ denotes the probability of choosing action $a$ in state $x$ by the student. In policy gradient methods, the policies are represented with a function approximator, usually modeled by a neural network, where we denote by $\theta\in\mathbb{R}^{N}$ and $\phi\in\mathbb{R}^{N}$ the corresponding policy weights of the student and teacher, respectively; the teacher and student parameterized policies are denoted by $\pi_T(\cdot|\cdot;\phi)$ and $\pi_S(\cdot|\cdot;\theta)$.  In what follows, we adapt this parameterization structure into the notation and interchangeably refer to $\pi_S$ and $\pi_T$ with their associated weights, $\theta$ and $\phi$.
	
	We consider the simulation optimization setting, where we can sample from the underlying MDP and observe the costs. Consider a possible state--action--cost trajectory $\tau\coloneqq\{x_0,a_0,c_0,x_1,a_1,c_1,\cdots,x_{H-1},a_{H-1},c_{H-1},x_H\}$ and let $\mathcal{T}\coloneqq\{\tau\}$ to be the set of all possible trajectories under all admissible policies.  For simplicity of exposition, we assume that the first hitting time $H$ of a terminal state $x_{term}$ from any given $x$ and following a stationary policy is bounded almost surely with an upper bound $H_{max}$, i.e., $H\leq H_{max}$ almost surely. Since the sample trajectories in many RL tasks terminate in finite time, this assumption is not restrictive. For example, the game fails after reaching a certain state or a time-out signal may terminate the trajectory. Along a trajectory $\tau$, the system incurs a discounted cost $J_{\theta}(\tau)=\sum_{t=0}^{H-1} \gamma^t c_t$, with discount factor $\gamma\in(0,1]$, and the probability of sampling such a trajectory is $\mathbb{P}_{\theta}(\tau)=P_0(x_0)\prod_{t=0}^{H-1}\pi_S(a_t|x_t;\theta)P(x_{t+1}|x_t,a_t)$. 
	We denote the expected cost from state $x$ onward until hitting the terminal state $x_{term}$ by $V_{\theta}(x)$, i.e., 
	\begin{align}
	V_{\theta}(x) = \mathbb{E}_{\theta}\left[\textstyle{\sum}_{t=0}^{H-1} \gamma^t C(x_t,a_t)|x_0 = x\right].
	\end{align}

	\subsection{Distance Measure} \vskip-5pt
An important question that arises is how to quantify the distance between the policies of the teacher and the student. There are several distance measures studied in the literature for computing the closeness of two probability measures. Among those, the \textit{Kullback-Leibler} (KL) divergence \cite{nasrabadi2007pattern} is a widely used metric.  In this work, we consider both \textit{reverse} (KL-R) and \textit{forward} (KL-F) KL-divergence, defined as
	\begin{align}\tag{KL-R}
	D_{KL}(\theta\;\|\; \phi) \coloneqq D_{KL}&(  \mathbb{P}_{\theta} (\tau)\;\|\;\mathbb{P}_{\phi} (\tau) ) =
	\textstyle{\sum}_{\tau\in\mathcal{T}} \mathbb{P}_{\theta}(\tau) \log\tfrac{\mathbb{P}_{\theta}(\tau)}{\mathbb{P}_{\phi}(\tau)}\label{eq:traj-wise-reverse}\\
\tag{KL-F}D_{KL}(\phi\;\|\; \theta) \coloneqq D_{KL}&(\mathbb{P}_{\phi} (\tau) \;\|\; \mathbb{P}_{\theta} (\tau) ) =
	\textstyle{\sum}_{\tau\in\mathcal{T}} \mathbb{P}_{\phi}(\tau) \log\tfrac{\mathbb{P}_{\phi}(\tau)}{\mathbb{P}_{\theta}(\tau)}.\label{eq:traj-wise-forward}
	\end{align}
	
	KL-divergence is known to be an asymmetric distance measure, meaning that changing the order of the student and teacher distributions will cause different learning behaviors. We will use the reverse KL-divergence in the theoretical analysis since it provides more compact notations. However, in all of the experiments, we will consider the forward setting, i.e., ${D}_{KL}(\phi \| \theta)$, unless otherwise specified. Informally speaking, this form of KL-divergence, which is also known as the \textit{mean-seeking} KL, allows the student to perform actions that are not included in the teacher's behavior. This is because \textit{i}) when the teacher can perform an action $a_t$ in a given state $s_t$, the student should also have $\pi_S(a_t|s_t)>0$ to keep the distance finite, and \textit{ii}) the student can have $\pi_S(a_t|s_t)>0$ irrespective of whether the teacher is doing that action or not. The reverse direction, $D_{KL}(\theta\|\phi )$, known as the \textit{mode-seeking} KL, can be useful as well. For example, let's assume that the teacher policy is a mixture of several Gaussian sub-policies. Using the reverse order will allow the student to assign only one sub-policy as its decision making policy. Hence, the choice of reverse KL would be preferred if the student wants to find a policy which is close to a teacher's sub-policy with the highest return. The justification of this behavior is also visible from the definition: when $\pi_S(a_t|s_t)>0$ for a given state and action, then the teacher also should have $\pi_T(a_t|s_t)>0$. Also, $\pi_T(a_t|s_t) = 0$ would not allow  $\pi_S(a_t|s_t) > 0$. For more detailed discussion and examples, we refer the interested reader to Appendix \ref{sec:measure} and Section 10.1.2 of \cite{nasrabadi2007pattern}.\vspace{-5pt}

	\subsection{Optimization Problems}
\vskip-5pt
The student's optimization problems that we would like to solve for reverse KL-divergence \eqref{opt-r} and the forward KL-divergence \eqref{opt-f} are defined as	\vspace{-.1cm}
    \begin{multicols}{2}
    \noindent
	\begin{align}\tag{OPT-R}
	\begin{split}
	\min_{\theta\in\Theta}&\;V_{\theta}(x_0)\\
	\text{s.t.}& \;D_{KL}(\theta\;\|\; \phi) \leq \delta\label{opt-r}
	\end{split}
	\end{align} 
    \columnbreak
    \noindent
	\begin{align}\tag{OPT-F}
	\begin{split}
	\min_{\theta\in\Theta}&\;V_{\theta}(x_0)\\
	\text{s.t.}& \;D_{KL}(\phi\;\|\;\theta ) \leq \delta, \label{opt-f}
	\end{split}
	\end{align} 
	\end{multicols}
	\vspace{-1em}
	\vspace{-.1cm}where $\delta$ is an upper bound on the KL-divergence and $\Theta$ is a convex compact set of possible policy parameters. Most of the theoretical analysis of the two optimization problems is quite similar, so we will use \eqref{opt-r} as our main formulation. In Appendix \ref{sec:equivalence}, we investigate the equivalence of both problems and state their minor differences.
	
	The widely adopted problem studied for MDPs only contains the objective function; however, we impose an additional constraint to restrict the student's policy. By fixing an appropriate value for $\delta$, one can enforce a constraint on the maximum allowed deviation of the student policy  from that of the teacher. The objective is to find a set of optimal points $\theta^*$ that minimizes the discounted expected cost while not violating the KL constraint. Notice that $\pi_S=\pi_T$ is a trivial feasible solution. In addition, we need to have the following assumption to ensure that \eqref{opt-r} is well-defined:
	\begin{assumption}\label{assumption:feas-optr}
		\textbf{(Well-defined \eqref{opt-r})} For any state--action pair $(x,a)\in\mathcal{X}\times\mathcal{A}$ with $\pi_T(x,a)=0$, we have $\pi_S(x,a)=0$.\vspace{-5pt}
	\end{assumption}
	Intuitively, Assumption \ref{assumption:feas-optr} specifies that when the teacher does not take a specific action in a given state, the student also cannot choose that action. Even though this assumption might seem restrictive, it is valid in situations in which the student is indeed limited to the positive-probability action space of the teacher.
Alternatively, we can certify this assumption by adding a small noise term to the outcome of the teacher's policy at the expense of some information loss. \vspace{-5pt}
    
	\subsection{Lagrangian Relaxation of \eqref{opt-r}}
	\vskip-5pt
	The standard method for solving \eqref{opt-r}  is by applying  Lagrangian relaxation \cite{bertsekas1999nonlinear}. We define the Lagrangian function
	\begin{align}
	    L(\theta,\lambda)
	    \coloneqq 
	    V_{\theta}(x_0) + \lambda \left({D}_{\text{KL}}(\theta\;\|\;\phi) -\delta\right),\label{eq:def-lag}
	\end{align}
	where $\lambda$ is the Lagrange multiplier. Then the optimization problem \eqref{opt-r} can be converted into the following problem:
	\begin{align}
	\textstyle{\max}_{\lambda\geq 0}\textstyle{\min}_{\theta \in \Theta} \; L(\theta,\lambda). \label{eq:opt-lag}
	\end{align}
	The intuition beyond \eqref{eq:opt-lag} is that we now allow the student to deviate arbitrarily much from the teacher's policy in order to decrease the cumulative cost, but we penalize any such deviation. 
	
	Next, we define a dynamical system which, as we will prove in Appendix \ref{sec:app-grad-conv}, solves problem \eqref{opt-r} under several common assumptions for stochastic approximation methods. Once we know the optimal Lagrange multiplier $\lambda^*$, then the student's optimal policy is
	\begin{equation}
\theta^* \in \textstyle{\argmin}_{\theta} V_{\theta}(x_0) + \lambda^* \left(D_{\text{KL}}(\theta\;\|\;\phi) -\delta\right).
	\end{equation} 
	A point $(\theta^*,\lambda^*)$ is a saddle point of $L(\theta,\lambda)$ if for some $r>0$, we have
	\begin{align}
	L(\theta^*,\lambda)\leq L(\theta^*,\lambda^*)\leq L(\theta,\lambda^*)
	\end{align}
	for all $\theta\in\Theta\cap\mathbb{B}_r(\theta^*)$ and $\lambda \geq 0$, where $\mathbb{B}_r(\theta^*)$ represents a ball around $\theta^*$ with radius $r$. Then, the saddle point theorem \cite{bertsekas1999nonlinear} immediately implies that $\theta^*$ is the local optimal solution of \eqref{opt-r}.\vspace{-5pt}

    \section{Primal--Dual Policy Gradient Algorithm}	\label{sec:pdpg}
    \vskip-5pt
    We propose a primal--dual policy gradient (PDPG) algorithm for solving \eqref{opt-r}. Due to space limitation, we leave the detailed algorithm to Appendix \ref{app:pdpg}, but the overall scheme is as follows. After initializing the student's policy parameters, possibly with those of the teacher, we sample multiple trajectories under the student's policy at each iteration $k$. Then the sampled trajectories are used to calculate the approximate  gradient of the Lagrangian function with respect to $\theta$ and $\lambda$. Finally, using an optimization algorithm, we update $\theta$ and $\lambda$  according to the approximated gradients.

	In order to prove our main convergence result, we need some technical assumptions on the student's policy and step sizes. 
	
	\begin{assumption}\label{assumption:differentiable}
		\textbf{(Smooth policy)} For any $(x,a)\in\mathcal{X}\times\mathcal{A}$, $\pi_S(a|x;\theta)$ is a continuously differentiable function in $\theta$ and its gradient is $\mathscr{L}$-Lipschitz continuous, i.e., for any $\theta^1$ and $\theta^2$,
		\begin{align}
		\left\|\nabla_\theta\pi_S(a|x;\theta)\bigr|_{\theta=\theta^1}-\nabla_\theta\pi_S(a|x;\theta)\bigr|_{\theta=\theta^2}\right\| \leq \mathscr{L}\|\theta^1- \theta^2\|.
		\end{align}
	\end{assumption}

	\begin{assumption}\label{assump:step}
	\textbf{(Step-size rules)} The step-sizes $\alpha_1(k)$ and $\alpha_2(k)$ in update rules \eqref{alg:eq:update-theta} and \eqref{alg:eq:update-lambda} satisfy the following relations:  
	\begin{enumerate}[label=(\textit{\roman*}),ref=(\roman*),noitemsep,topsep=0pt]
		\item \label{assump:step:1} $\sum_{k} \alpha_1(k) = \infty ; \; \sum_{k} \alpha_1^2(k)<\infty$,
		\item \label{assump:step:2} $\sum_{k} \alpha_2(k) = \infty ; \; \sum_{k} \alpha_2^2(k)<\infty$,
		\item \label{assump:step:3} $\alpha_2(k) = o(\alpha_1(k))$. \vspace{-5pt}
	\end{enumerate} 
	\end{assumption}
	Relations \ref{assump:step:1} and \ref{assump:step:2} in Assumption \ref{assump:step} are common in stochastic approximation algorithms, and \ref{assump:step:3} indicates that the Lagrange multiplier update is in a slower time-scale compared to the policy updates. The latter condition simplifies the convergence proof by allowing us to study the PDPG as a two-time-scale stochastic approximation algorithm. The following theorem states the main theoretical result of this paper.

	\begin{theorem}\label{thm:main}
		Under Assumptions \ref{assumption:feas-optr}, \ref{assumption:differentiable},  and \ref{assump:step}, the sequence of policy updates (starting from $\theta^0$ sufficiently close to a local optimum point $\theta^*$) and Lagrange multipliers converges almost surely to a saddle point of the Lagrangian, i.e., $(\theta(k),\lambda(k)) \overset{a.s.}{\longrightarrow} (\theta^* , \lambda^*$). Moreover, $\theta^*$ is a local optimal solution of \eqref{opt-r}. 
	\end{theorem}
	\begin{proof}
	\vskip-5pt
		(\textit{sketch}) The proof is similar to those found in \cite{tamar2012policy,chow2017risk}. It is based on representing $\theta$ and $\lambda$ update rules with a two-time-scale  stochastic approximation algorithm. For each timescale, the algorithm can be shown to converge to the stationary points of the corresponding continuous-time system. Finally, it can be shown that the fixed point is, in fact, a locally optimal point. In Appendix \ref{sec:app-conv}, we provide a formal proof of this theorem.
	\end{proof}
	\begin{corollary}\label{cor-stationary}
	Under Assumptions \ref{assumption:feas-optr}, \ref{assumption:differentiable},  and \ref{assump:step}, the sequence of policy updates and Lagrange multipliers converges globally to a stationary point of the Lagrangian almost surely. Moreover, if $\theta^*$ is in the interior of $\Theta$, then $\theta^*$ is a feasible first order stationary point of \eqref{opt-r}, i.e., $\nabla_\theta V_\theta(x_0)|_{\theta=\theta^*}=0$ and $D_{KL}(\theta^*\;\|\;\phi)\leq \delta$.
	\end{corollary}

	Theorem \ref{thm:main} and Corollary \ref{cor-stationary} are also valid for the forward KL constraint case, as we discuss in Appendix \ref{sec:equivalence}.\vspace{-5pt} 

	\section{Practical PDPG Algorithm }\label{sec:prac-main}
\vskip-5pt	 Although the algorithm presented in the previous section is proved to converge to a first-order stationary point, it cannot directly serve as a practical learner algorithm. The main reason is that it produces a high-variance approximation of the gradients, which would lead to unstable learning. In this section, we propose several approximations to the theoretically-justified PDPG in order to develop a more practical algorithm. {For this algorithm, we will consider the forward definition of KL-divergence due to the mean-covering property}.
	
	One source of variance is the reward bias, which can be handled by adding a critic, similar to \cite{konda2000actor}.  Our next adjustment is to use an approximation of the step-wise KL-divergence, defined as 
	\begin{align}
	    \hat{D}^{step}_{KL} (\phi \;\|\; \theta) = \tfrac{1}{N} \textstyle{\sum}_{j=1}^N \textstyle{\sum}_{t=0}^H {D}^{step}_{KL}\bigl(\pi_T(\cdot|x_t;\phi)\;\|\;\pi_S(\cdot|x_t;\theta)\bigr) \label{eq:kl-step-app}  
	 \end{align}
	 where 
	    ${D}^{step}_{KL}\bigl(\pi_T(\cdot|x_t;\phi)\;\|\;\pi_S(\cdot|x_t;\theta)\bigr) = \textstyle{\sum}_{a\in\mathcal{A}}  \pi_T(a|x_t;\phi) \log\tfrac{\pi_T(a|x_t;\phi)}{\pi_S(a|x_t;\theta)}\nonumber.$
	As we discuss in Appendix \ref{sec:measure}, using \eqref{eq:kl-step-app} results in a much smaller variance, while still ensuring the convergence results. Intuitively, this equation suggests that instead of computing the trajectory probabilities and then computing the KL-divergence, as in \eqref{eq:traj-wise-reverse}, one can compute the KL in every visited state along a trajectory and sum them up. In addition to this change, we will further normalize each ${D}^{step}_{KL}$ by its trajectory length $H$ to remove the effect of the variable horizon length. The latter modification will lead to more sensible KL values and will make the choice of $\delta$ easier.

    A second difficulty with the algorithm in Section \ref{sec:pdpg} is that, unlike conventional policy gradient algorithms, there is no guarantee that the student's optimal policy is a deterministic one. In fact, in most of our experiments, it happens that the optimal policy is stochastic, especially when the teacher's policy itself is stochastic. To illustrate this, consider two scenarios: \textit{i}) The student refuses to do the suggested action of a deterministic teacher. In this case, she would incur an infinite cost as a result of her disobedience, so the problem will be infeasible. \textit{ii}) The teacher is less informative and has no clue about most of the state space, so often takes random actions. Trying to emulate this teacher would cause degraded performance for the student as well, so the student would also take many less informed actions. 
    
    A stochastic optimal policy is usually not desirable since it poses major safety and reliability challenges, so our next adjustments are an attempt to address this issue. One possible mitigation for the first scenario might be using a bounded distance measure such as Hellinger \cite{cramer1946mathematical} instead of KL-divergence, but our numerical experiments did not confirm that this is effective. We observe that by using the Hellinger constraint, the total entropy of the student's policy stays high, without any improvement in the student's policy. Instead, we propose using \textit{percentile KL-clipping}, which is defined as
    \begin{align}
        \textsc{clip}_\rho\left({D}^{step}_{KL}\right) = \max \{\rho\text{\% percentile of all ${D}^{step}_{KL}$s at time $t$},{D}^{step}_{KL}\} 
    \end{align}
    
    In fact, the \textsc{clip} function enables the student to totally disagree with the teacher in $\rho$\% of the visited states, without receiving an extremely large penalty. Selecting the values for $\rho$ depends on our perception about how perfect the teacher is. Setting $\rho$ close to 100 means that we believe in the teacher's suggestions. As we decrease $\rho$, we rely less on the teacher and can disobey more freely.
    
    The last major modification is to control the expected entropy at a certain small level $\delta^{ent}>0$, i.e.,
    \begin{align}
        ent(\theta) \coloneqq - \mathbb{E}_{x}\left[ \tfrac{1}{H}\textstyle{\sum}_{t=0}^H \textstyle{\sum}_{a\in\mathcal{A}}  \pi_S(a|x;\theta) \log{\pi_S(a|x;\theta)}\right]=\delta^{ent}. \label{eq:ent-constraint}
    \end{align}
    The justification for adding \eqref{eq:ent-constraint} is that we would like the optimal policy to be close to a deterministic one as much as possible. By setting a small value for $\delta^{ent}$, we can enforce this property. Also, this constraint tries to avoid having a deterministic policy in intermediate training steps, in order to allow more exploration. To add this constraint, we use the same Lagrangian technique, adding an extra term to the Lagrangian function:
    \begin{align}
        L(\theta,\lambda,\zeta)
	    \coloneqq 
	    V_{\theta}(x_0) + \lambda \left({D}_{\text{KL}}(\phi\;\|\;\theta) -\delta\right) + \zeta \left(ent(\theta) -\delta^{ent} \right).
    \end{align}{}
    All of these modifications, along with a few others, are summarized in Algorithm \ref{alg:ppg} of Appendix \ref{app:ppg}.\vspace{-5pt}

	\section{Experiments}
\vskip-5pt	We illustrate the efficiency of the proposed methods with multiple GridWorld experiments. In the first set of experiments, the teacher tries to teach the student to perform an oscillating maneuver around the walls. In the second set, we study how the student can comprehend changes in the environment change and utilize them to increase its rewards.\vspace{-5pt}
    
    \subsection{Square-Wave Teacher}
\vskip-5pt    In this experiment, we consider a teacher who gives a suggestion in every state of a GridWorld. We study two variants of the teacher, one who is very determined about all of his suggestions and the other who is less confident. Figure \ref{fig:env_zig} illustrates the environment and both teachers' suggestions. A student wants to find a path from the {blue} state to the {green} target. Each step has a reward $-1$ and reaching the target brings $+100$ reward. If the student wants to act independently, the optimal path is a trivial horizontal line. However, our objective is to force the student to ``listen'' to her teacher up to some level.

		\begin{figure*}[htbp]
		\centering
		\begin{subfigure}[t]{0.32\textwidth}
			\centering
			\includegraphics[width=\textwidth,trim=.1cm .1cm .1cm .1cm,clip]{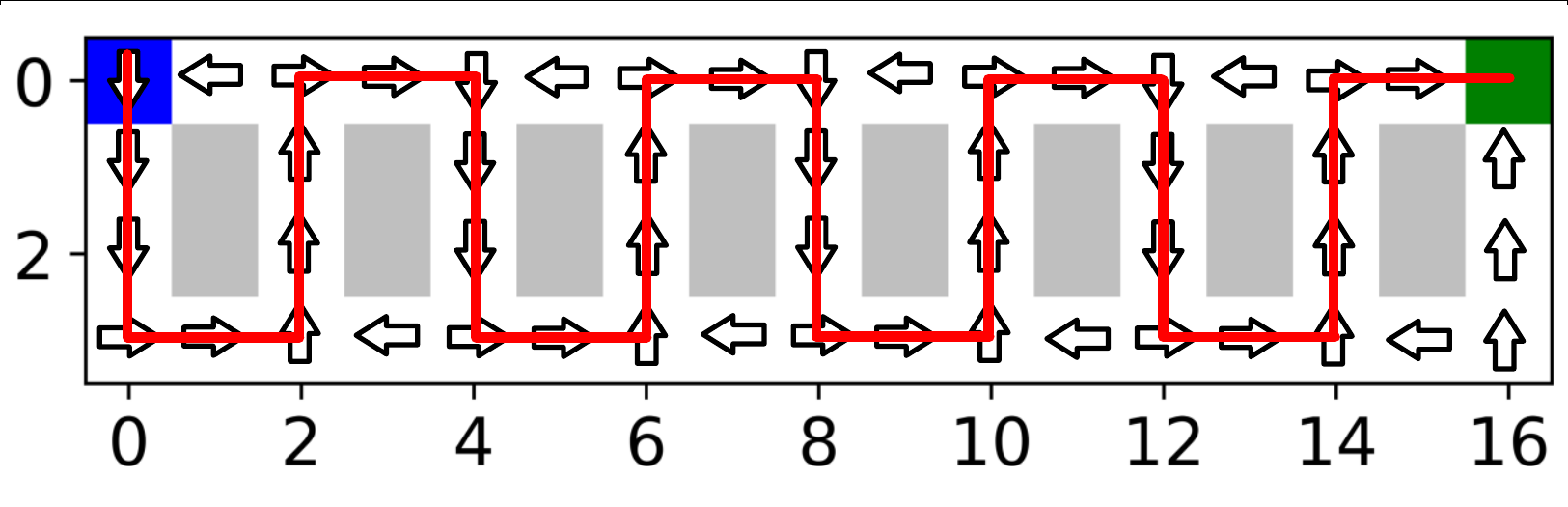}
			\caption{``Determined'' teacher with the corresponding suggested actions for every state. The red line shows the teacher's suggested path to target }
			\label{fig:env_zig_1}
		\end{subfigure}%
		\hspace{.1cm}
		\begin{subfigure}[t]{0.32\textwidth}
			\centering
			\includegraphics[width=\textwidth,trim=.1cm .1cm .1cm .1cm,clip]{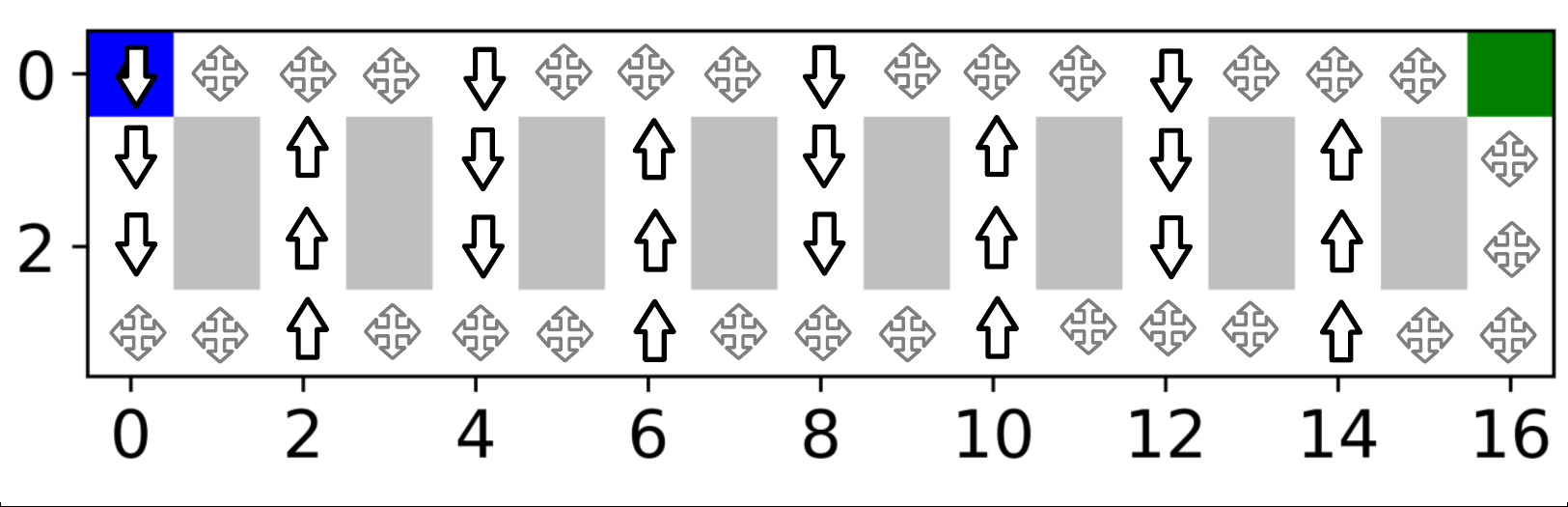}
			\caption{``Less confident'' teacher with deterministic actions in a subset of states and uniformly random actions in the rest}
			\label{fig:env_zig_2-1}
		\end{subfigure}%
	\hspace{.1cm}
		\begin{subfigure}[t]{0.32\textwidth}
			\centering
			\includegraphics[width=\textwidth,trim=.1cm .1cm .1cm .1cm,clip]{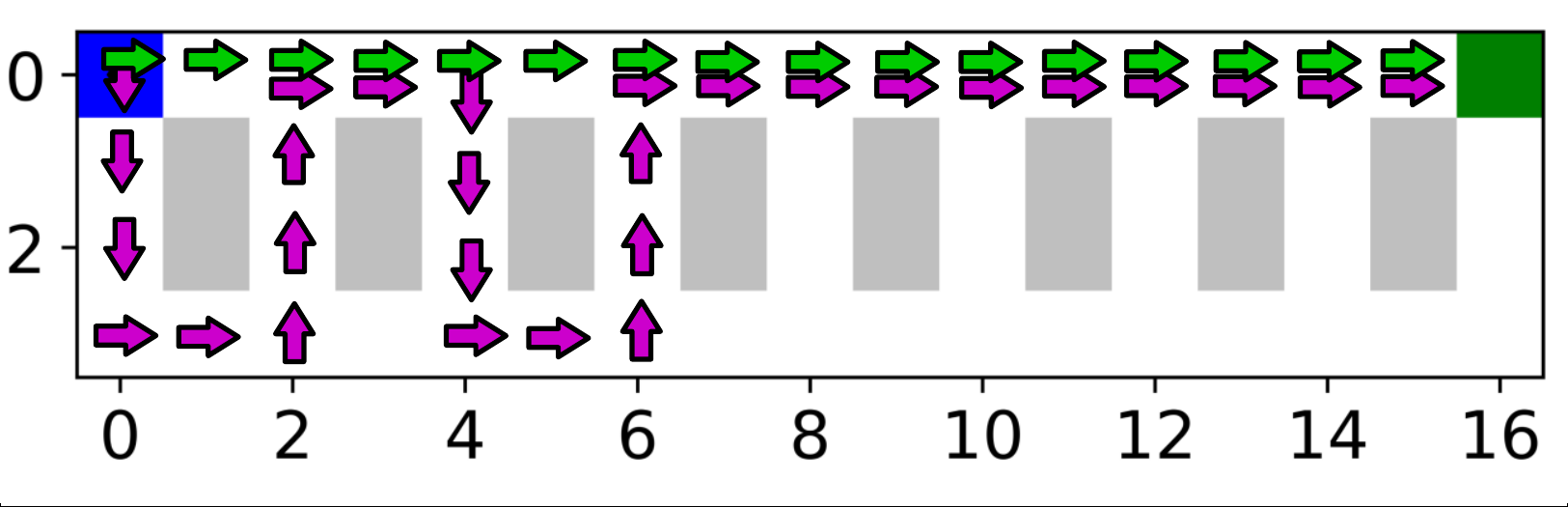}
			\caption{Optimal Path (in green) versus a sample path found by PDPG (in purple) with $\delta=0.2$ and $\rho=8$}
			\label{fig:env_zig_1_opt}
		\end{subfigure}
		
		\caption{Two different teachers with suggested actions and optimal path.} 
		\label{fig:env_zig}
		\vskip-5pt
	\end{figure*}
	
	\textit{Determined Teacher}:
	In this part, a teacher has a preferred action in every state with a probability of around 98\%. As we observe in Figure \ref{fig:env_zig_1}, these suggestions might help the student in reaching the target, but they are inefficient. For instance, if the student follows a square-wave sequence of actions as illustrated by the red line, she will be able to reach the target while following all of the teacher's suggestions exactly. Our objective is to allow the student deviate from the teacher for a few steps, to find shorter routes.

	By using PDPG, the student is able to find policies that are a mixture of the horizontal path and the square-wave route. For example, in Figure \ref{fig:env_zig_1_opt}, we have illustrated an instance of the student's optimized path with $\delta=0.2$ . The extent to which either policy is followed depends on values of $\delta$ and the KL-clipping parameter $\rho$. Figure  \ref{fig:rew-delta} illustrates the student's total reward for different $\delta$ quantities without KL-clipping. We observe that as we increase $\delta$, we allow the student to act more freely, hence she gets a higher reward. However, after 5000 training iterations,  the reward remains at the same level with too much oscillation. Recalling the discussions of Section \ref{sec:prac-main}, this behavior indicates convergence to a stochastic policy. 

	To reduce the oscillating behavior, we proposed adding an entropy constraint and KL-clipping. Figure \ref{fig:ent-2} shows how adding the entropy constraint results in a more deterministic (i.e., lower entropy) policy. Also, in Figure \ref{fig:rew-rho}, we have added KL-clipping. As we decrease $\rho$, the student can totally disagree with the teacher in a larger proportion of the visited states, so she can find better policies with higher rewards. For different values of $\rho$, we see that the policy can converge to either a stochastic or deterministic one. For $\rho=70,75$, it converges to a deterministic horizontal line policy. With  $\rho=80,85$, it learns to deterministically follow one $\sqcup$-shaped path followed by a horizontal route, and for $\rho=95$, it follows a $\mathrlap{\mathrlap{\sqcup}\hspace{.18cm}\sqcap}\hspace{.36cm}\sqcup$-shaped path with a horizontal line at the end. Notice that even with $\rho=100$, which means no clipping, the student is not exactly following the teacher. We also observe that for $\rho=90$, it fails to converge to a deterministic policy. One justification for such a failure is that the student's policy is far better than the less-rewarding deterministic one, but not good enough to get to the next level of performance. Finally, Figure \ref{fig:conv-lam} shows how the $\lambda$ and $\zeta$ values  converge to their optimal values.

	\textit{Less Confident Teacher}:
	This experiment is designed to illustrate how a less confident teacher can still teach the student to follow some of his suggestions, but it will yield a lower level of confidence of the student. Figure \ref{fig:env_zig_2-1} shows the suggested actions of the teacher;  he is deterministic only in a subset of the states. For the rest, he does not have any information, so he suggests actions uniformly at random.  The less confident teacher still has the square wave as the general idea (which is bad, just like the determined teacher), but also has extra randomness that points the student in even worse directions. In other words, the less confident teacher has a worse policy overall than the determined one. Recommending random actions causes the student to have more volatile behavior. We can observe this fact by comparing Figure \ref{fig:part_rew-delta} with \ref{fig:rew-delta}, where the student's converged policy produces a wider range of rewards for the less-confident teacher's case. Also, the average reward for this case is slightly lower, which can be explained by the inadequate information that the less-confident teacher provides for solving the task. \vspace{-5pt}
	
	Figure \ref{fig:part_rew-rho} shows that adding KL-clipping helps in reducing the volatility, but one needs to choose a much smaller value for $\rho$ (compare it with Figure \ref{fig:rew-rho}). Yet, even a small $\rho$ does not necessarily result in a deterministic policy; for $\rho$ as small as $0.4$, the student has converged to a stochastic one.
    	\begin{figure*}[htbp]
		\centering
\vskip-5pt
		\begin{subfigure}[t]{0.23\textwidth}
			\centering
			\includegraphics[width=\textwidth,trim=.1cm .1cm .1cm .1cm,clip]{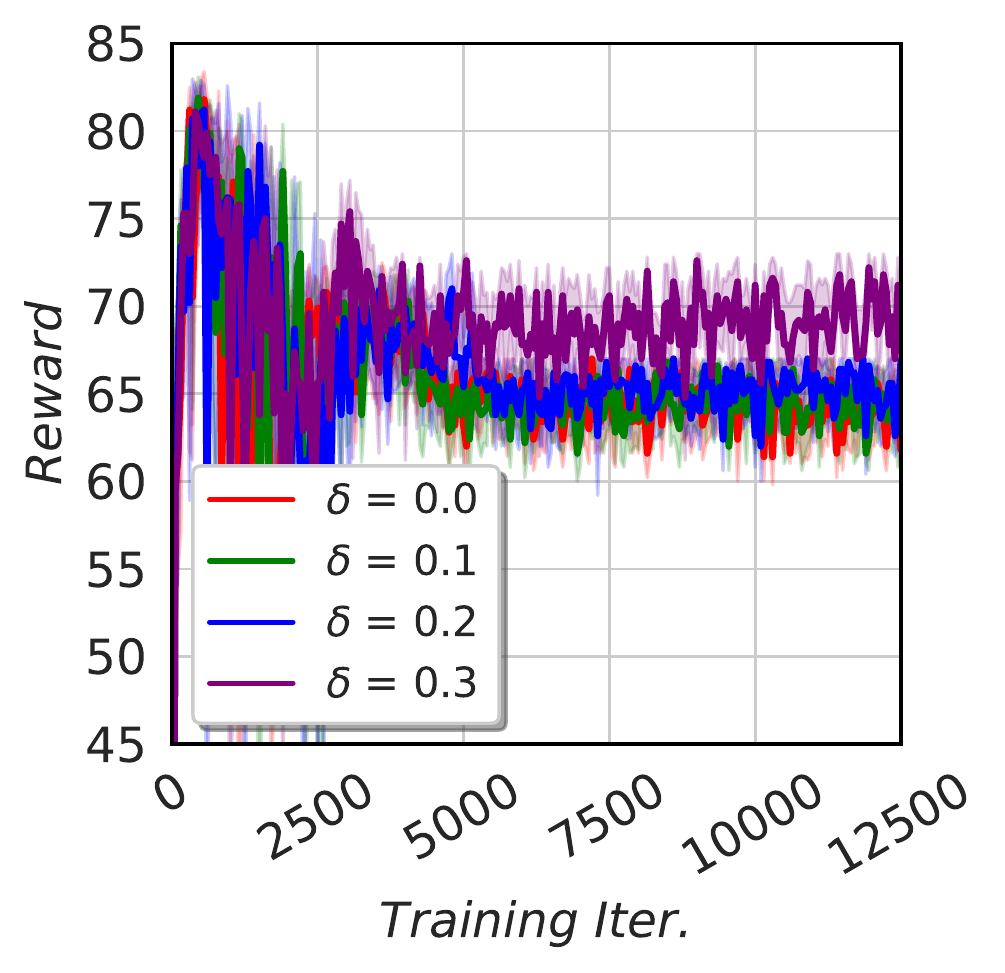}
			\caption{The effect of $\delta$ on reward; no KL-clipping}
			\label{fig:rew-delta}
		\end{subfigure}
		\hspace{.1cm}
		\begin{subfigure}[t]{0.23\textwidth}
			\centering
			\includegraphics[width=\textwidth,trim=.1cm .1cm .1cm .1cm,clip]{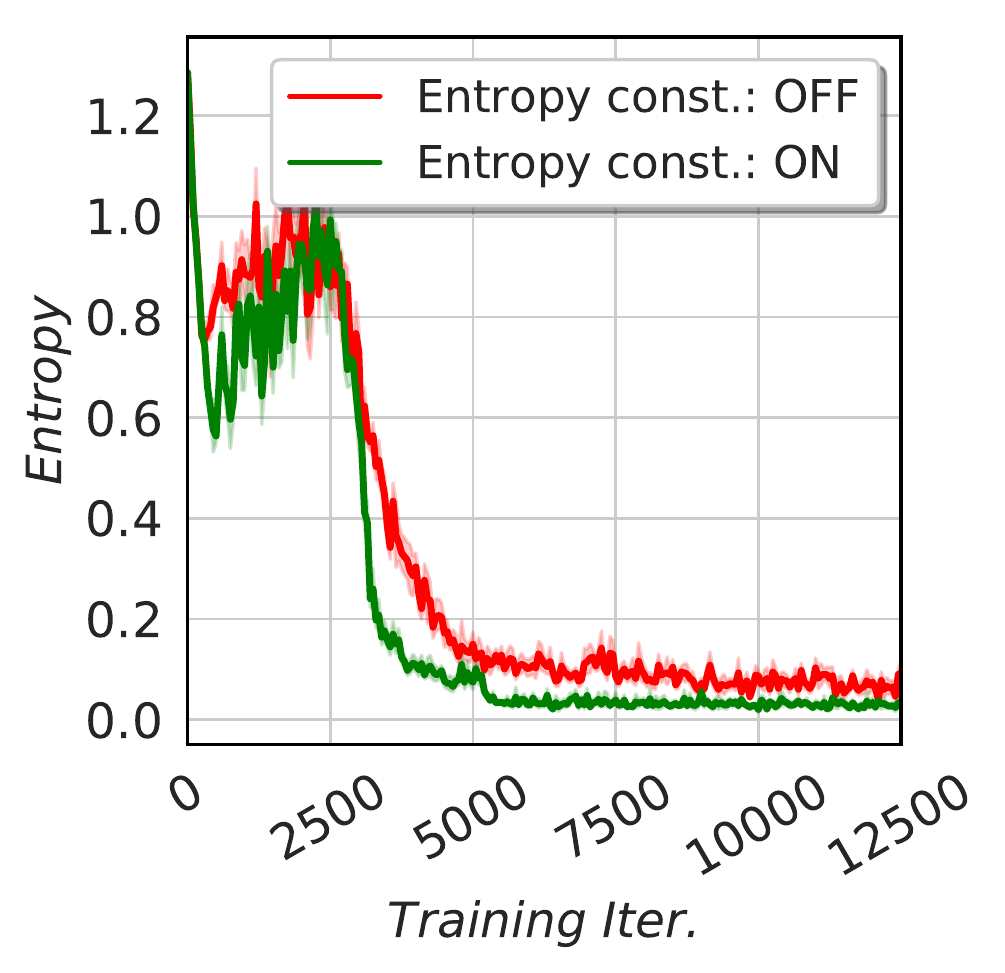}
			\caption{The effect of entropy constraint; $\delta=0.2$}
			\label{fig:ent-2}
		\end{subfigure}
		\hspace{.1cm}
		\begin{subfigure}[t]{0.23\textwidth}
			\centering
			\includegraphics[width=\textwidth,trim=.1cm .1cm .1cm .1cm,clip]{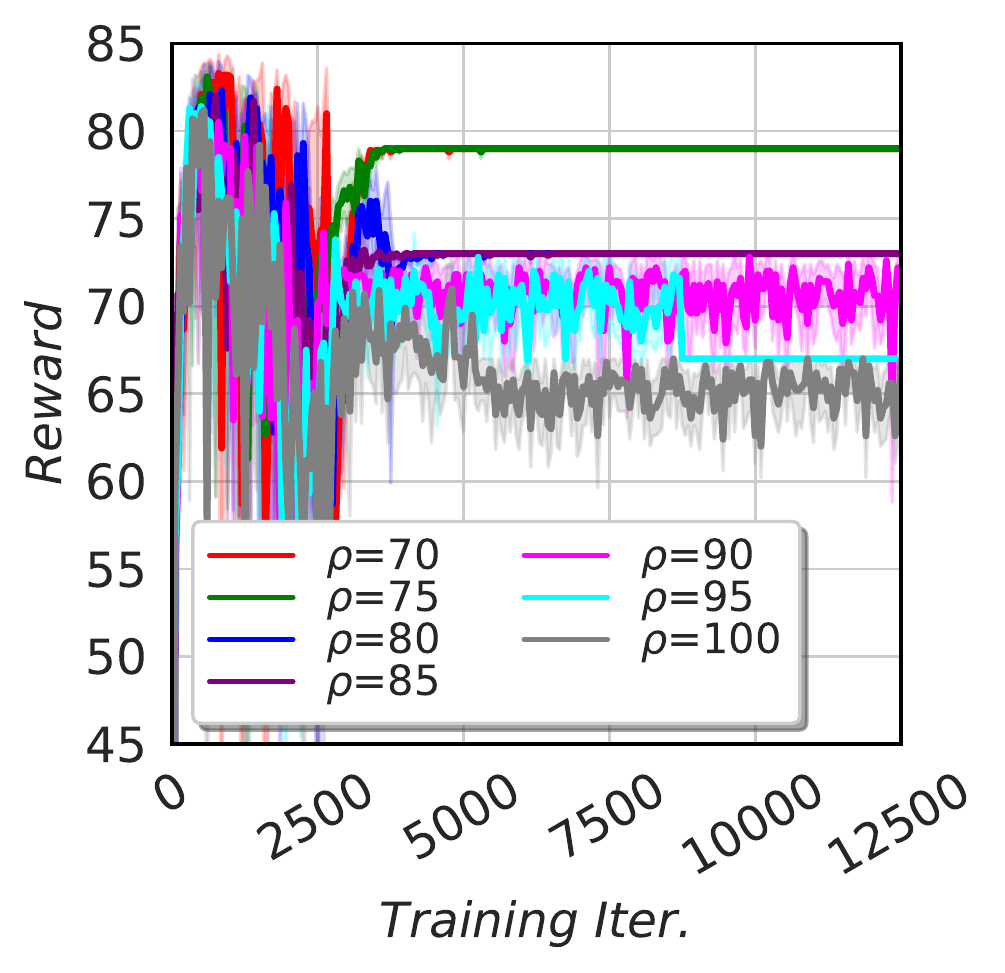}
			\caption{Total reward for different $\rho$ and $\delta=0.2$}
			\label{fig:rew-rho}
		\end{subfigure}
		\hspace{.1cm}
		\begin{subfigure}[t]{0.23\textwidth}
			\centering
			\includegraphics[width=\textwidth,trim=.1cm .1cm .1cm .1cm,clip]{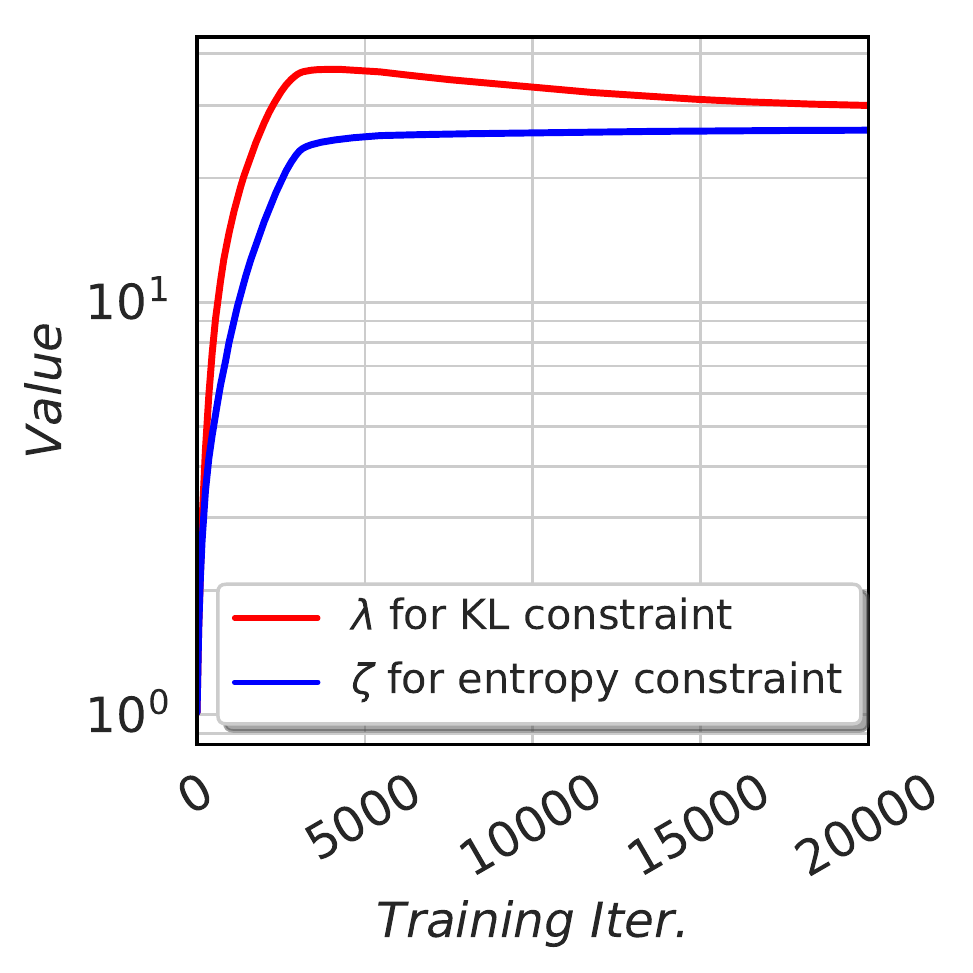}
			\caption{Convergence of $\lambda$ and $\zeta$; $\delta=0.2$ }
			\label{fig:conv-lam}
		\end{subfigure}%
		\caption{Performance of a student learning from the deterministic teacher}
		\label{fig:env_zig_res}
    \begin{minipage}[t]{0.49\textwidth}
		\centering
		\begin{subfigure}[t]{0.48\columnwidth}
			\centering
			\includegraphics[width=\columnwidth,trim=.1cm .1cm .1cm .1cm,clip]{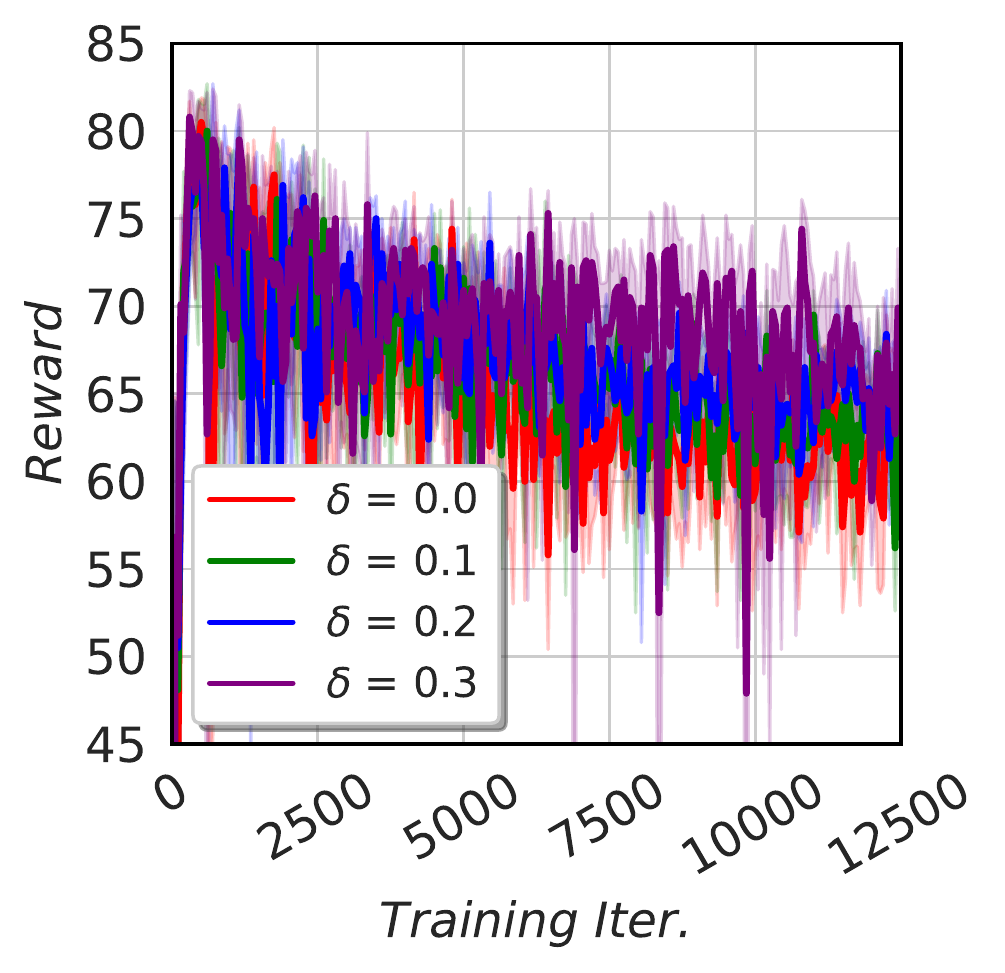}
			\caption{The effect of $\delta$ on reward; no KL-clipping}
			\label{fig:part_rew-delta}
		\end{subfigure}
		\hspace{.1cm}
		\begin{subfigure}[t]{0.48\columnwidth}
			\centering
			\includegraphics[width=\columnwidth,trim=.1cm .1cm .1cm .1cm,clip]{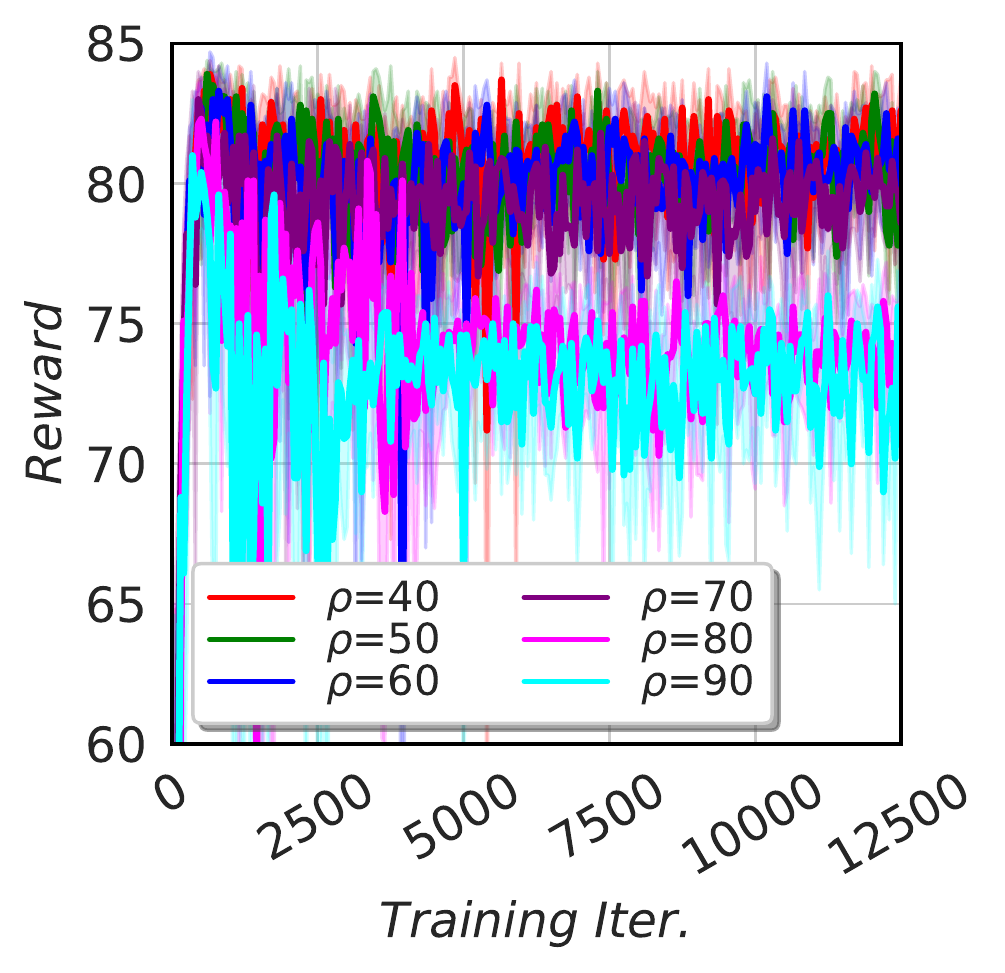}
			\caption{Total reward for different $\rho$ and $\delta=0.2$}
			\label{fig:part_rew-rho}
		\end{subfigure}
		\caption{Performance of a student learning from the less confident teacher}
		\label{fig:part_env_zig_res}
		\end{minipage}
		\hfill
	    \begin{minipage}[t]{0.49\textwidth}
		\centering
		\begin{subfigure}[t]{0.46\columnwidth}
			\centering
			\includegraphics[width=\columnwidth,trim=.1cm .1cm .1cm .1cm,clip]{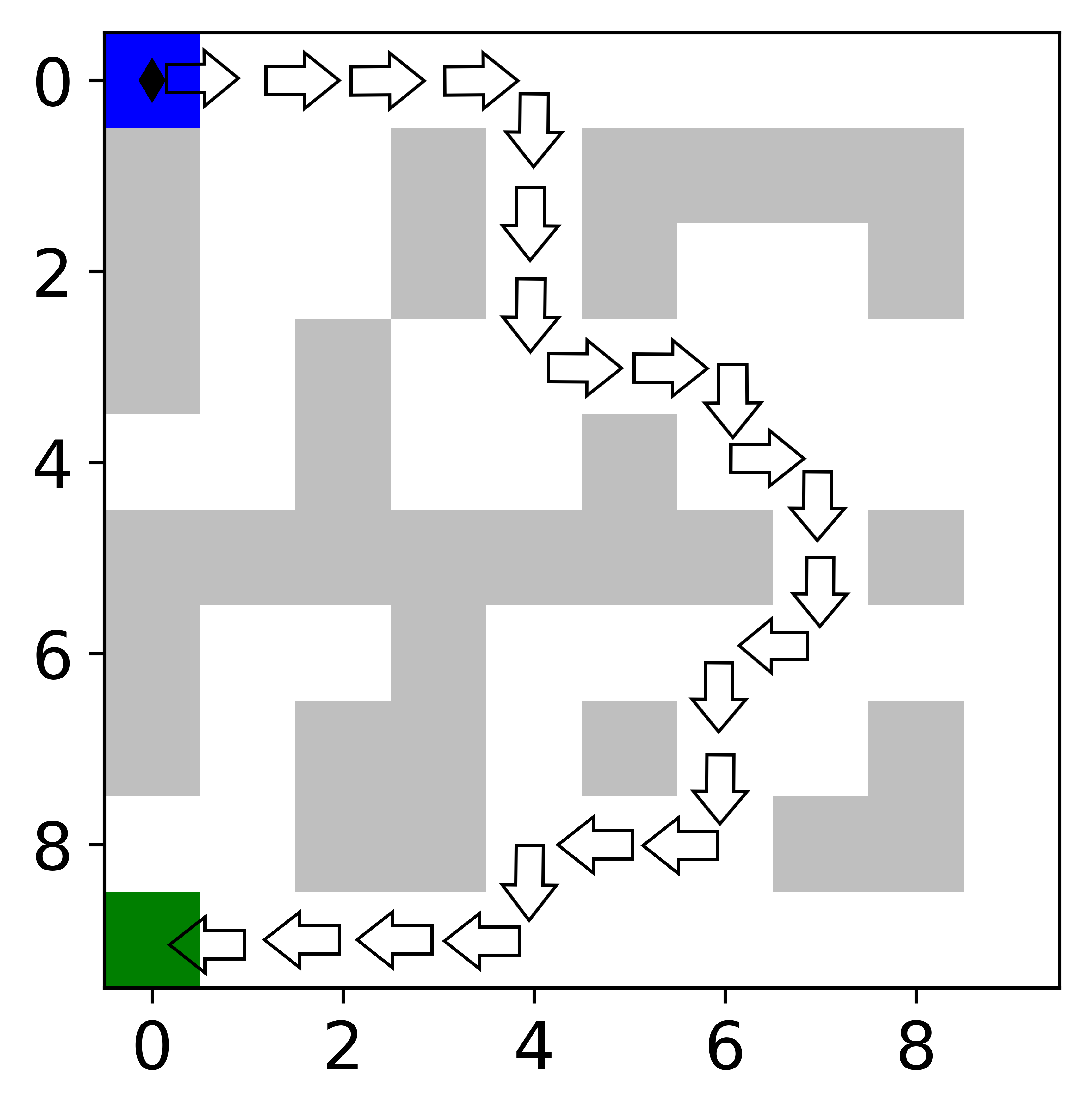}
			\caption{Teacher's environment. The optimal path found by the RL is demonstrated}
			\label{fig:env_T}
		\end{subfigure}%
	\hspace{.1cm}
		\begin{subfigure}[t]{0.46\columnwidth}
			\centering
			\includegraphics[width=\columnwidth,trim=.1cm .1cm .1cm .1cm,clip]{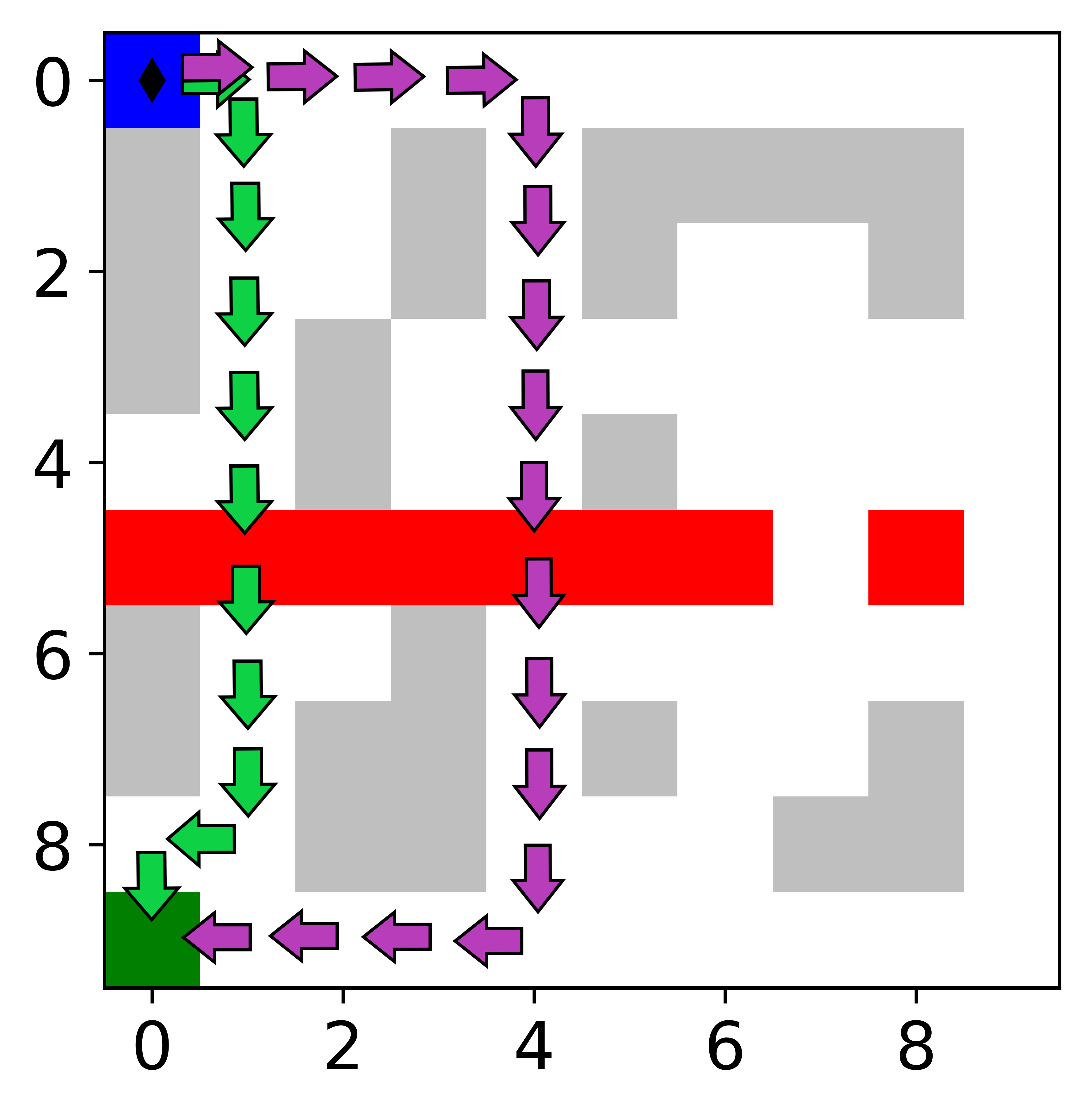}
			\caption{Student's environment. She can leap the red walls at a penalty. The paths found by RL (in green) and by PDPG (in purple) are illustrated}
			\label{fig:env_S}
		\end{subfigure}
		
		\caption{Teacher and student environments as well as their policies}
		\label{fig:env_zip_part}
		\end{minipage}
	\end{figure*}
	
	\subsection{Wall Leaping}
	\vskip-5pt
    The purpose of this experiment is to show that PDPG can act as an improvement method when the student encounters a slightly modified environment. The teacher's reward structure is similar to the structure in the previous experiment, i.e., $-1$ for every step and $+100$ for reaching the target. However, the student comprehends that she can leap over some of the walls with a reward of $-2$. We use a vanilla policy gradient algorithm to train the teacher, which provides paths like the one illustrated in Figure \ref{fig:env_T}. If we allow the student to learn without any constraint, it will find the green path in Figure \ref{fig:env_zip_part} with a KL-divergence of $\approx 0.89$. However, this is not what we are looking for since it is extremely different from the teacher. Instead, we use the PDPG algorithm to constrain the policy deviation with $\delta=0.3$. Using this parameter, the student learns to follow the purple path, with a KL-divergence of $\approx 0.23$.  \vspace{-5pt}

	\section{Related Work}
	\vskip-5pt
 Learning from a teacher is a well-studied problem in the literature on supervised learning  \cite{girshick2014rich} and imitation learning \cite{schaal1999imitation,thomaz2006reinforcement}. However, we are not aware of any work using a teacher to control specific behaviors of a student. The typical use case of a student--teacher framework in RL is in ``policy compression,'' where the objective is to train a student from a collection of well-trained RL policies. Policy distillation \cite{rusu2015policy} and actor--mimic \cite{parisotto2015actor}  are two methods that distill the trained RL agents, in a supervised learning fashion, into a unified policy of the student. In contrast, we follow a completely distinct objective, where a student is continually interacting with an environment and it only uses the teacher's signals as a guideline for shaping her policy.

Closest to ours, \citet{schmitt2018kickstarting} propose ``kickstarting RL,'' a method that uses the teacher's information for better training. Incorporating the idea of population-based training, they design a hand-crafted decreasing schedule of Lagrange multipliers, $\{\lambda^k\}\to 0$. Nevertheless, the justification for such a schedule is not clearly visible. However, noticing that their problem is a special case of ours with $\delta = \infty$, our findings confirm the credibility of their approach, i.e., our findings indicate that $\lambda^* = \lim_{k\to\infty}\lambda_{min}^k = 0$ according to strong duality. This observation also conforms with the experimental findings of \cite{schmitt2018kickstarting}, and our theoretical results indicate that when there is no obligation on being similar to the teacher, the student is better off eventually operating independently.  Similarly, their method only uses the teacher for faster learning.

    Imposing certain constraints on the behavior of a policy is also a common problem in the context of ``safe RL'' \cite{achiam2017constrained,leike2017ai, chow2018lyapunov}. Typically, these problems look for policies that avoid hazardous states either during training or execution. Our problem is different in that we follow another type of constraint, yet similar methods might be applied. Using a domain-specific programming language instead of neural networks can be an alternative method to add interpretability \cite{verma2018programmatically}, but it lacks the numerous advantages inherent in end-to-end and differentiable learning. In an alternative direction, it is also possible to manipulate the policy shape by introducing auxiliary tasks or reward shaping \cite{jaderberg2016reinforcement}. Despite the simplicity of the latter approach, it has a very limited capability. For example, it is unclear how  reward shaping can suggest directions similar to our square-wave teacher. In summary, we believe that our end-to-end method, by implicitly adding interpretable components, can partially alleviate the concerns related to the RL policies. \vspace{-5pt}

    \section{Concluding Remarks}
    \vskip-5pt
    In this paper, we introduce a new paradigm called corrective RL, which allows a ``student'' agent to learn to optimize its own policy while also staying sufficiently close to the policy of a ``teacher.'' Our approach is motivated by the fact that practitioners may be reluctant to adopt the policies proposed by RL algorithms if they differ too much from the status quo. Even if the RL policy produces an impressive expected return, this may not be satisfactory evidence to switch the operation of a billion-dollar company to a policy found by an RL. We believe that corrective RL provides a straightforward remedy by constraining how far the new policy can deviate from the old one or another desired, target policy. Doing so will help reduce the stresses of adopting a novel policy. 
    

    We believe that, with further extensions, corrective RL has the potential to address to some of RL's interpretability challenges. Using more advanced optimization algorithms, studying different distance measures, considering continuous-action problems, and having multiple teachers represent fruitful avenues for future research.

 \section*{Acknowledgement}
This work was partially supported by the U.S. National Science Foundation, under award numbers NSF:CCF:1618717, NSF:CMMI:1663256 and NSF:CCF:1740796, and XSEDE IRI180020.

\bibliographystyle{unsrtnat}
\bibliography{references}

\newpage

    \appendix
    
    \section{PDPG Algorithms} 
        \subsection{Computing the Gradients}\label{sec:app-grad}

     The Lagrangian function in the optimization problem \eqref{eq:opt-lag} can be re-written as
     \begin{align}
          L(\theta,\lambda) =& \sum_{\tau\in\mathcal{T}} \mathbb{P}_\theta (\tau) J(\tau) +  \lambda  \sum_{\tau\in\mathcal{T}} \mathbb{P}_\theta (\tau) \log\frac{\mathbb{P}_\theta (\tau)}{\mathbb{P}_\phi (\tau)}- \lambda\delta\nonumber\\
          =&\sum_{\tau\in\mathcal{T}} \mathbb{P}_\theta (\tau) \left(J(\tau) +  \lambda  \log\frac{\mathbb{P}_\theta (\tau)}{\mathbb{P}_\phi (\tau)}\right)- \lambda\delta. \label{eq:lag-refed}
     \end{align}
     Recall that $\mathcal{T}$ is the set of all trajectories under all admissible policies. By taking the gradient of $L(\theta,\lambda)$ with respect to $\theta$, we have:
     \begin{align}
         \nabla_\theta L(\theta,\lambda) &= \sum_{\tau\in\mathcal{T}} \nabla_\theta \mathbb{P}_\theta (\tau) \left(J(\tau) +  \lambda  \log\frac{\mathbb{P}_\theta (\tau)}{\mathbb{P}_\phi (\tau)}\right) +\mathbb{P}_\theta (\tau) \left(\lambda\nabla_\theta\log \mathbb{P}_\theta (\tau)\right)\nonumber\\
         &=\sum_{\tau\in\mathcal{T}} \mathbb{P}_\theta (\tau)\nabla_\theta \log\mathbb{P}_\theta (\tau) \left(J(\tau) +  \lambda  \log\frac{\mathbb{P}_\theta (\tau) }{\mathbb{P}_\phi (\tau)}+ \lambda\right)\nonumber \\
         &=\mathbb{E}_{\mathcal{T}}\left[ \nabla_\theta \log\mathbb{P}_\theta (\tau) \left(J(\tau) +  \lambda  \log\frac{\mathbb{P}_\theta (\tau) }{\mathbb{P}_\phi (\tau)}+ \lambda\right)\right],\label{eq:nabla-lag-theta}
     \end{align}
     and the term $\nabla_\theta \log\mathbb{P}_\theta (\tau)$ can be simplified as
     \begin{align}
         \nabla_\theta \log\mathbb{P}_\theta (\tau) &= \nabla_\theta \left(\log P_0(x_0) + \sum_{t=0}^{H-1}\log P(x_{t+1}|x_t,a_t) + \sum_{t=0}^{H-1} \log \pi_S(a_t|x_t;\theta) \right) \nonumber\\
         &=\sum_{t=0}^{H-1} \nabla_\theta \log \pi_S(a_t|x_t;\theta)\nonumber\\
         &=\sum_{t=0}^{H-1} \frac{\nabla_\theta \pi_S(a_t|x_t;\theta)}{\pi_S(a_t|x_t;\theta)}.\label{eq:log_prob_p_pi}
     \end{align}
     The gradient of $L(\theta,\lambda)$ with respect to $\lambda$ is
     
     \begin{align}
         \nabla_\lambda L(\theta,\lambda) &=  \sum_{\tau\in\mathcal{T}} \mathbb{P}_\theta (\tau) \log\frac{\mathbb{P}_\theta (\tau)}{\mathbb{P}_\phi (\tau)}- \delta = D_{\text{KL}}(\mathbb{P}_\theta (\tau)\;\|\;\mathbb{P}_\phi (\tau))- \delta. \label{eq:nabla-lag-lambda}
     \end{align}
     
     By using a set of sample trajectories $\{\tau_j, j=1,\ldots,N\}$ generated under the student policy, one can approximate the gradients \eqref{eq:nabla-lag-theta} and \eqref{eq:nabla-lag-lambda} as 
     \begin{align}
         \nabla_\theta L(\theta,\lambda) &\approx \frac{1}{N}\sum_{j=1}^{N}\left[ \nabla_\theta \log\mathbb{P}_\theta (\tau_j) \left(J(\tau_j) +  \lambda  \log\frac{\mathbb{P}_\theta (\tau_j) }{\mathbb{P}_\phi (\tau_j)}+ \lambda\right)\right],\nonumber\\
         \nabla_\lambda L(\theta,\lambda) &\approx \hat{D}_\text{KL}(\theta\;\|\;\phi) - \delta =  \frac{1}{N}\sum_{j=1}^{N} \log\frac{\mathbb{P}_\theta (\tau)}{\mathbb{P}_\phi (\tau)}- \delta ,\nonumber  
     \end{align}
     which are the update rules that will be used later on, in \eqref{alg:eq:update-theta} and \eqref{alg:eq:update-lambda}.
     
    \subsection{ PDPG Algorithm  }\label{app:pdpg}
Having derived the gradients of the Lagrangian {(in Appendix \ref{sec:app-grad})}, we have all the necessary information for proposing our primal--dual policy gradient (PDPG) algorithm, which is described in Algorithm \ref{alg:PG}. 
    
    	\begin{algorithm}[htbp]
		\caption{Primal-Dual Policy Gradient (PDPG) Algorithm for \eqref{opt-r}}
		\label{alg:PG}
		\begin{algorithmic}[1]
			\STATE \textbf{input:} teacher's policy with weights $\phi$
			\STATE \textbf{initialize:} student's policy with $\theta^0$, possibly equal to $\phi$; initialize step size schedules $\alpha_1(\cdot)$ and $\alpha_2(\cdot)$
			\WHILE {TRUE}
			\FOR {$k = 0,1,\ldots$}
			\STATE following policy $\theta^k$, generate a set of $N$ trajectories $\mathcal{T}^k=\{\tau_j^k,\,j =1,2,\ldots,N \}$, each starting from an initial state $x_0 \sim P_0(\cdot)$
			\STATE \textbf{($\theta$-update)} update $\theta^k$ according to  \label{alg:pg:theta-step}
			\begin{align}
			\theta^{k+1} = \Gamma_{\Theta}\Bigl[\theta^{k} - \alpha_1(k)\Bigl( \frac{1}{N}\sum_{j=1}^{N} \nabla_{\theta}\log{\mathbb{P}_\theta(\tau^k_j)}\bigr|_{\theta=\theta^k}\bigl(J(\tau^k_j) + \lambda^k \log\frac{\mathbb{P}_\theta(\tau^k_j)}{\mathbb{P}_\phi(\tau^k_j)}+\lambda^k \bigr)\Bigr)\Bigr]\label{alg:eq:update-theta}
			\end{align}
			\STATE \textbf{($\lambda$-update)} update $\lambda^k$ according to
			\begin{align}
			\lambda^{k+1} = \Gamma_{\Lambda}\Bigl[\lambda^{k} + \alpha_2(k) \Bigl(\frac{1}{N}\sum_{j=1}^{N} \log\frac{\mathbb{P}_\theta(\tau^k_j)}{\mathbb{P}_\phi(\tau^k_j)}  -\delta \Bigr)\Bigr]\label{alg:eq:update-lambda}
			\end{align}
			\ENDFOR
			
			\IF {$\lambda^k$ converges to $\lambda_{max}$} \label{alg:pg:lambdamaxbeg}
			\STATE $\lambda_{max} \leftarrow 2 \lambda_{max}$
			\ELSE {}
			\STATE return $\theta$ and $\lambda$; break
			
			\ENDIF\label{alg:pg:lambdamaxend}
			
			\ENDWHILE
		\end{algorithmic}
	\end{algorithm}
    {After initializing the student with the teacher's policy, at each iteration $k$, we take a mini-batch of sample trajectories under the student's policy $\theta^k$}.  { In step \ref{alg:pg:theta-step}, we use the sampled trajectories to compute an  approximate gradient of the Lagrangian function with respect to $\theta$ and update the policy parameters} {in the negative direction of the approximate gradient} with step size $\alpha_1(k)$. In addition to policy parameter updates, the dual variables are learned concurrently using the recursive formula 
	\begin{align}
	\lambda^{k+1} = \lambda^{k} + \alpha_2(k) \left[ \hat{D}_\text{KL}(\theta\;\|\;\phi) - \delta \right],
	\end{align}
    where $\alpha_2(k)$ represents the associated step-size rule. 
    
    In this algorithm, we need to use two projection operators to ensure the convergence of the algorithm. Specifically, $\Gamma_\Theta$ is an operator that projects $\theta$ to the closest point in $\Theta$, i.e., $\Gamma_\Theta(\theta) = \argmin_{\hat{\theta}\in \Theta} \|\theta - \hat{\theta}\|^2$. Similarly, $\Gamma_\Lambda$ is an operator that maps $\lambda$ to the interval $\Lambda\coloneqq [0,\lambda_{max}]$. Finally, in steps \ref{alg:pg:lambdamaxbeg}--\ref{alg:pg:lambdamaxend}, we check whether $\lambda$ has converged to some point on the boundary. Such a convergence means that the projection space for the Lagrange multipliers is small, so we increment the upper bound and repeat searching for a better policy. 
	 
	\section{Convergence Analysis of PDPG for \textbf{\eqref{opt-r}}}\label{sec:app-grad-conv}

     Before starting the proof of Theorem \ref{thm:main}, noting the definition of $\nabla_\theta L(\theta,\lambda)$ 
     and $\nabla_\lambda L(\theta,\lambda)$
     , one can make the following observations:
     \begin{lemma}\label{lemma:lips}
     Under Assumption \ref{assumption:differentiable}, the following holds:
     \begin{enumerate}[label=\roman*)]
         \item \label{lem-lip-1} $\nabla_\theta \log\mathbb{P}_\theta (\tau)$ is Lipschitz continuous in $\theta$, which further implies that
         \begin{align}
             \|\nabla_\theta \log\mathbb{P}_\theta (\tau)\|^2\leq \kappa_1(\tau) \left(1+\|\theta\|^2\right) \label{eq:lem-lip-1}
         \end{align}
         for some $\kappa_1(\tau)<\infty$.
         \item \label{lem-lip-2} $\nabla_\theta L(\theta,\lambda)$ is Lipschitz continuous in $\theta$, which further implies that
          \begin{align}
             \|\nabla_\theta L(\theta,\lambda)\|^2\leq \kappa_2 \left(1+\|\theta\|^2\right)\label{eq:lem-lip-2}
         \end{align}
         for some constant $\kappa_2<\infty$.
         \item \label{lem-lip-3} $\nabla_\lambda L(\theta,\lambda)$ is Lipschitz continuous in $\lambda$.
     \end{enumerate}
     \end{lemma}
     
 \begin{proof} Recall from \eqref{eq:log_prob_p_pi} that $\nabla_\theta \log\mathbb{P}_\theta (\tau) =\sum_{t=0}^{H-1} {\nabla_\theta \pi_S(a_t|x_t;\theta)}/{\pi_S(a_t|x_t;\theta)}$ whenever  $\pi_S(a_t|x_t;\theta)>\psi$ for all $t$ and for some $\psi>0$. Assumption \ref{assumption:differentiable} indicates that $\nabla_\theta \pi_S(a_t|x_t;\theta)$ is $\mathscr{L}$-Lispchitz continuous in $\theta$. Then using the fact the sum of the product of (bounded) Lipschitz functions is Lipschitz itself, one can conclude the Lipschitz continuity of $\nabla_\theta \log\mathbb{P}_\theta (\tau)$, and we denote by $L_1$ its finite Lipschitz constant. Also, noting that $H<\infty$ w.p.\ 1, then $\nabla_\theta \log\mathbb{P}_\theta (\tau)<\infty$ w.p.\ 1.  The Lipschitz continuity implies that for any fixed $\theta_0\in\Theta$,
     \begin{align}
         \|\nabla_\theta \log\mathbb{P}_\theta (\tau)\| \leq \|\nabla_\theta \log\mathbb{P}_\theta (\tau) |_{\theta=\theta_0}\| + L_1 \|\theta-\theta_0\|\leq K_1(\tau)(1+\|\theta\|). \label{eq:lemma:lips:1}
     \end{align}
     The first inequality follows from the linear growth condition of Lipschitz functions and the last one holds for a suitable value of $K_1(\tau)\coloneqq\max\{L_1,\|\nabla_\theta \log\mathbb{P}_\theta (\tau) |_{\theta=\theta_0}\| + L_1 \|\theta_0\|\}<\infty$. Taking the square of both sides of \eqref{eq:lemma:lips:1} yields \eqref{eq:lem-lip-1} with $\kappa_1(\tau) \coloneqq 2 (K_1(\tau))^2<\infty$.
     
     Since $\mathbb{P}_\theta (\tau)$ and $\log \mathbb{P}_\theta (\tau)$ are continuously differentiable in $\theta$ whenever $\mathbb{P}_\theta (\tau)>0$, 
     the Lipschitz continuity of $\nabla_\theta L(\theta,\lambda)$ can be investigated, from its definition \eqref{eq:nabla-lag-theta}, as the sums of products of (bounded) Lipschitz functions. From the definition \eqref{eq:nabla-lag-theta}, and recalling Assumption \ref{assumption:feas-optr} and the compactness of $\Theta$, one can verify the validity of \eqref{eq:lem-lip-2} with 
     \begin{align}
         \kappa_2 = \mathbb{E}_\tau\left[\kappa_1(\tau) \left(\frac{C_{max}}{1-\gamma} + \lambda_{max}\max_{\theta\in\Theta}\log\frac{\mathbb{P}_\theta (\tau)}{\mathbb{P}_\phi (\tau)}\right) \right]<\infty.
     \end{align}
     Finally, \textit{\ref{lem-lip-3}} immediately follows from the fact that $\nabla_\lambda L(\theta,\lambda)$ is a constant function of $\lambda$.
     \end{proof}

     \subsection{Convergence of PDPG Algorithm}\label{sec:app-conv}
     We use the standard procedure for proving the convergence of the PDPG algorithm. The proof steps are common for stochastic approximation methods and we refer the reader to \cite{chow2017risk,bhat2009natural} and references therein for more details. We summarize the scheme of the proof in the following steps:
     \begin{enumerate}
         \item \textbf{Tracking o.d.e.}: Under Assumption \ref{assump:step}, one can view the PDPG as a two-time-scale stochastic approximation method. Then, using the results of Section 6 of \cite{borkar2009stochastic}, we show that the sequence of $(\theta^k,\lambda^k)$ converges almost surely to a stationary point $(\theta^*,\lambda^*)$ of the corresponding continuous-time dynamical system.
         \item \textbf{Lyapunov Stability}: By using Lyapunov analysis, we show that the continuous-time system is locally asymptotically stable at a first-order stationary point. 
         \item \textbf{Saddle Point Analysis}: Since we have used the Lagrangian as the Lyapunov function, it implies the system is stable in the stationary point of the Lagrangian, which is, in fact, a local saddle point. Finally, we show that with an appropriate initial policy, the policy converges to a local optimal solution $\theta^*$ for the \ref{opt-r}.
     \end{enumerate}
     
     First, let us denote by $\Psi_\Xi\left[f(\xi)\right]$ the right directional derivative of $\Gamma_{\Xi}(\xi)$ in the direction of $f(\xi)$, defined as
	\begin{align}
	    \Psi_\Xi\left[f(\xi)\right] \coloneqq \lim_{\alpha\downarrow 0}\frac{\Gamma_\Xi\left[\xi + \alpha f(\xi)\right] -\Gamma_\Xi\left[\xi\right] }{\alpha}\nonumber
	\end{align}
	for any compact set $\Xi$ and $\xi\in \Xi$.

     Since $\theta$ converges on a faster time-scale than $\lambda$ by Assumption \ref{assump:step}, one can write the $\theta$-update rule \eqref{alg:eq:update-theta} with a relation that is invariant to $\lambda$:
     \begin{equation*}
			\theta^{k+1} = \Gamma_{\Theta}\left[\theta^{k} - \alpha_1(k)\left(\left. \frac{1}{N}\sum_{j=1}^{N} \nabla_{\theta}\log{\mathbb{P}_\theta(\tau^k_j)}\right|_{\theta=\theta^k}\left(J(\tau^k_j) + \lambda \log\frac{\mathbb{P}_\theta(\tau^k_j)}{\mathbb{P}_\phi(\tau^k_j)}+\lambda \right)\right)\right]\nonumber.
	\end{equation*}

    Consider the continuous-time dynamics of $\theta\in\Theta$ defined as
	\begin{align}
	    \dot{\theta} = \Psi_\Theta\left[-\nabla_\theta L(\theta,\lambda)\right],\label{eq:update-theta-cont}
	\end{align}
	where by using the right directional derivative $\Psi_\Theta\left[-\nabla_\theta L(\theta,\lambda)\right]$ in the gradient descent algorithm for $\theta$, the gradient will point in the descent direction of $L(\theta,\lambda)$ along the boundary of $\Theta$ (denoted by $\partial \Theta$) whenever the $\theta$-update hits the boundary. We refer the interested reader to Section 5.4 of \cite{borkar2009stochastic} for discussions about the existence of the limit in \eqref{eq:update-theta-cont}.
	
	Since $\lambda$ converges in the slowest time-scale, the $\lambda$-update rule \eqref{alg:eq:update-lambda} can be re-written for a converged value $\theta^*(\lambda)$ as
	\begin{align}
			\lambda^{k+1} = \Gamma_{\Lambda}\left[\lambda^{k} + \alpha_2(k) \left(\frac{1}{N}\sum_{j=1}^{N} \log\frac{\mathbb{P}_{\theta^*(\lambda)}(\tau^k_j)}{\mathbb{P}_\phi(\tau^k_j)}  -\delta \right)\right].\nonumber
	\end{align}
    Consider the continuous-time dynamics corresponding to $\lambda$, i.e.
    \begin{align}
        \dot{\lambda} = \Psi_\Lambda\left[\nabla_\lambda L(\theta,\lambda)\right],\label{eq:update-lambda-cont}
    \end{align}
    where by using $\Psi_\Lambda\left[\nabla_\lambda L(\theta,\lambda)\right]$ in the gradient ascent algorithm, the gradient will point in the ascent direction along the boundary of $\Lambda$ (denoted by $\partial \Lambda$) whenever the $\lambda$-update hits the boundary.
    
    We prove Theorem \ref{thm:main} next.
    \begin{proof}
    \textbf{Convergence of the $\theta$-update}:
    First, we need to show that the assumptions of Lemma 1 in Chapter 6 of \cite{borkar2009stochastic} hold for the  $\theta$-update and an arbitrary value of $\lambda$. 
    Let us justify these assumptions: (\textit{i}) the Lipschitz continuity follows from Lemma \ref{lemma:lips}, and (\textit{ii}) the step-size rules follow from Assumption \ref{assump:step}. \text{(\textit{iii})} For an arbitrary value $\lambda$, one can write the $\theta$-update as a stochastic approximation, i.e.,
	\begin{align}
		\theta^{k+1} = \Gamma_{\Theta}\left[\theta^{k} + \alpha_1(k)\left( -\nabla_\theta L(\theta,\lambda)|_{\theta = \theta^k}+ M_{\theta_{k+1}}\right)\right],\label{eq:proof-theta-update}
	\end{align}
	where
	\begin{align}
	M_{\theta^{k+1}} = \nabla_\theta L(\theta,\lambda)|_{\theta = \theta^k} - \frac{1}{N}\sum_{j=1}^{N} \left.\nabla_{\theta}\log{\mathbb{P}_\theta(\tau^k_j)}\right|_{\theta=\theta^k}\left(J(\tau^k_j) + \lambda^k \log\frac{\mathbb{P}_\theta(\tau^k_j)}{\mathbb{P}_\phi(\tau^k_j)}+\lambda^k \right).
	\end{align}
	For $M_{\theta^{k+1}}$ to be a Martingale difference error term, we need to show that its expectation with respect to the filtration $\mathcal{F}_\theta^k = \sigma(\theta^{m},M_{\theta^{m}}, m\leq k)$ is zero and that it is square integrable with $\mathbb{E}\left[\|M_{\theta^{k+1}}\|^2|\mathcal{F}_\theta^k\right]\leq \kappa^k(1+\|\theta^k\|^2)$ for some $\kappa^k$. Since the trajectories $\mathcal{T}^k$ are generated from the probability mass function $\mathbb{P}_{\theta^k}(\cdot)$, it immediately follows that $\mathbb{E}\left[M_{\theta^{k+1}}|\mathcal{F}_\theta^k\right] = 0$. Also, we have:
	\begin{align}
	    &\|M_{\theta^{k+1}}\|^2\nonumber\\
	    &\leq 2\| \nabla_\theta L(\theta,\lambda)|_{\theta 
	    = \theta^k}\|^2 + \frac{2}{N^2}\left(\frac{C_{max}}{1-\gamma}+ \lambda_{max}\left( D_{max}^k+1\right)\right)^2\left\|\sum_{j=1}^{N} \nabla_{\theta}\log{\mathbb{P}_\theta(\tau^k_j)}\bigr|_{\theta=\theta^k}\right\|^2\nonumber\\
	    &\leq 2\kappa_2^k \left(1+\|\theta^k\|^2\right)+ \frac{2^N}{N^2} \left(\frac{C_{max}}{1-\gamma}+ \lambda_{max}\left( D_{max}^k+1\right)\right)^2\left(\sum_{j=1}^{N} \kappa_1^k(\tau^k_j) \left(1+\|\theta^k\|^2\right)\right)\nonumber\\
	    &\leq \kappa^k \left(1+\|\theta^k\|^2\right),\nonumber
	\end{align}
	where
	\begin{align}
	    D_{max}^k &=\max_{1\leq j\leq N}\log\frac{\mathbb{P}_\theta(\tau^k_j)}{\mathbb{P}_\phi(\tau^k_j)},\quad \text{and} \nonumber\\
	    \kappa^k &= 2\kappa_2^k+ \frac{2^N}{N} \max_{1\leq j\leq N}\kappa_1^k(\tau^k_j) \left(\frac{C_{max}}{1-\gamma}+ \lambda_{max}\left( D_{max}^k+1\right)\right)^2< \infty.\nonumber
	\end{align}
	 The first and second inequality uses the relation $\|\sum_{i=1}^N a_i\|^2\leq 2^{N-1}(\sum_{i=1}^N\|a\|^2 )$. Also, the second one uses the results of Lemma \ref{lemma:lips}. Finally, the boundedness of $\kappa^k$ follows from Assumption \ref{assumption:feas-optr} and having $\kappa^k_1<\infty$, $\kappa^k_2(\tau^k_j)<\infty$ w.p. 1. Finally, \text{(\textit{iv})} $\sup_k \|\theta^k\|<\infty$ almost surely, because all $\theta^k$ are within the compact set $\Theta$. Hence, by Theorem 2 of Chapter 2 in \cite{borkar2009stochastic}, the sequence $\{\theta^k\}$ converges almost surely to a (possibly sample path dependent) internally chain transitive invariant set of o.d.e. \eqref{eq:update-theta-cont}.

 	For a given $\lambda$, define the Lyapunov function
	\begin{align}
	    \mathcal{L}_\lambda(\theta) = L(\theta,\lambda) - L(\theta^*,\lambda),
	\end{align}
	where $\theta^*\in\Theta$ is a local minimum point. For the sake of simplifying the proof, let us consider that $\theta^*$ is an isolated local minimum point, i.e., there exists $r$ such that for all $\theta\in\mathbb{B}_{r}(\theta^*)$, $\mathcal{L}_\lambda(\theta)>\mathcal{L}_\lambda(\theta^*) $. This means that the Lyapunov function $\mathcal{L}_\lambda(\theta)$ is locally positive definite, i.e., $\mathcal{L}_\lambda(\theta^*)=0$ and $\mathcal{L}_\lambda(\theta)>0$ for $\mathbb{B}_{r}\setminus\{\theta^*\}$. 
	
	If we establish the negative semi-definiteness of $d \mathcal{L}_\lambda(\theta)/{dt}\leq 0$, then we can use the Lyapunov stability theorems to show the convergence of the dynamical system. Consider the time derivative of corresponding continuous-time system for $\theta$, i.e.,
	 \begin{align}
	     \frac{d \mathcal{L}_\lambda(\theta)}{dt} = \frac{d L(\theta,\lambda)}{dt} = (\nabla_\theta L(\theta,\lambda))^T \Psi_\Theta(-\nabla_\theta L(\theta,\lambda)).
	 \end{align}
	Consider two cases:
	\begin{enumerate}[label=\textit{\roman*})]
	    \item For a fixed $\theta_0\in\Theta$, there exists $\alpha_0>0$ such that the update $\theta_0-\alpha\nabla_\theta L(\theta,\lambda)|_{\theta=\theta_0}\in\Theta$ for all $\alpha\in (0,\alpha_0]$. In this case, $\Psi_\Theta(-\nabla_\theta L(\theta,\lambda)) = -\nabla_\theta L(\theta,\lambda)$, which further implies that 
	    \begin{align}
	        \frac{d L(\theta_0,\lambda)}{dt} = -\|\nabla_\theta L(\theta,\lambda)|_{\theta=\theta_0}\|^2 \leq 0,\nonumber
	    \end{align}
	    and this quantity is non-zero as long as $\|\Psi_\Theta(-\nabla_\theta L(\theta,\lambda))\|\neq 0$.
	\item For fixed $\theta_0\in\Theta$ and any $\alpha_0>0$, there exists $\alpha\in(0,\alpha_0]$ such that $\theta_\alpha\coloneqq\theta_0-\alpha \nabla_\theta L(\theta,\lambda)|_{\theta=\theta_0}\not\in\Theta$. The projection $\Gamma_\Theta(\theta_\alpha) = \argmin_{\theta\in\Theta} \frac12\|\theta-\theta_\alpha\|^2$ maps $\theta_\alpha$ to a point in $\partial \Theta$. This projection is single-valued because of the compactness and convexity of $\Theta$, and we denote the projected point by $\bar{\theta}_\alpha\in\Theta$. Consider $\alpha\downarrow 0$, then
	\begin{align}
	    (\nabla_\theta L(\theta,\lambda))^T \Psi_\Theta(-\nabla_\theta L(\theta,\lambda)) &=  \lim_{\alpha\downarrow 0} \frac{ (\theta - \theta_\alpha)^T(\bar{\theta}_\alpha - \theta)}{\eta} \nonumber\\
	    &= \lim_{\alpha\downarrow 0} \frac{ - \|\bar{\theta}_\alpha-\theta|^2}{\eta^2} + \frac{ (\bar{\theta}_\alpha - \theta_\alpha)^T(\bar{\theta}_\alpha - \theta)}{\eta^2}\leq  0,\nonumber
	\end{align}
	where the last inequality follows from the Projection Theorem (see Proposition 1.1.9 of \cite{bertsekas2009convex}). Again, one can verify that the time-derivative quantity is non-zero as long as $\|\Psi_\Theta(-\nabla_\theta L(\theta,\lambda))\|\neq 0$.
	\end{enumerate}
	
	In summary, $d \mathcal{L}_\lambda(\theta)/{dt}\leq 0$ and this quantity is nonzero as long as $\|\Psi_\Theta(-\nabla_\theta L(\theta,\lambda))\|\neq 0$. Then by \textit{LaSalle's Local Invariant Set Theorem} (see, e.g., Theorem 3.4 of \cite{slotine1991applied}), we conclude that the dynamical system tends to the largest positive invariant set within $\mathbb{M}_\theta\coloneqq\{\theta: \|\Psi_\Theta(-\nabla_\theta L(\theta,\lambda))\| = 0\}$. Notice that $\theta^*\in\mathbb{M}_\theta$. Let $l>0$ be equal to $ \min\{\mathcal{L}_\lambda(\theta): \|\Psi_\Theta(-\nabla_\theta L(\theta,\lambda))\| = 0, \theta\in \mathbb{B}_{r}(\theta^*)\setminus {\theta^*}\}$. Then every trajectory starting from the attraction region $\{\theta\in\mathbb{B}_{r}(\theta^*)| \mathcal{L}_\lambda(\theta)<l\}$ will tend to the local minimum $\theta^*$. Since we chose $\theta^*$ to be arbitrary, this holds for all local minima. Hence, using Corollary $4$ of Chapter 2 in \cite{borkar2009stochastic}, we conclude that if the initial policy $\theta^0$ is within the attraction region of a local minimum point $\theta^*$, then it will converge to it almost surely. 
	
	\begin{remark}
		The case in which $\theta^*$ is not isolated can be handled similarly, with the minor difference that the convergence happens to a set of optimal points instead of to a single point.
	\end{remark}

	\textbf{Convergence of the $\lambda$-update}: We need to show that the assumptions of Theorem 2 in Chapter 6 of \cite{borkar2009stochastic} hold for the two-time-scale stochastic approximation theory. Let us verify the validity of these assumptions: (\textit{i}) $\nabla_\lambda L(\theta,\lambda)$ is a Lipschitz function in $\lambda$ from Lemma \ref{lemma:lips}, and (\textit{ii}) step-size rules follow from Assumption \ref{assump:step}. (\textit{iii})	Since $\lambda$ converges in a slower time-scale, we have $\|\theta^{k,i} - \theta^*(\lambda^k)\|\to 0$ almost surely as $i\to \infty$, which, according to the Lipschitz continuity of $\nabla_\lambda L(\theta,\lambda)$, implies that
	\begin{align}
	    \|\nabla_\lambda L(\theta,\lambda)|_{\theta=\theta^{k,i},\lambda=\lambda^k} - \nabla_\lambda L(\theta,\lambda)|_{\theta=\theta^*(\lambda^k),\lambda=\lambda^k}\|\to 0 \quad\text{ as } i\to \infty.
	\end{align}
	Hence the $\lambda$-update can be written as 
	\begin{align}
	\lambda^{k+1} = \Gamma_{\Lambda}\Bigl[\lambda^{k} + \alpha_2(k) \left( \nabla_\lambda L(\theta,\lambda)|_{\theta=\theta^*(\lambda^k),\lambda=\lambda^k} + M_{\lambda^{k+1}}\right)\Bigr]\nonumber,
	\end{align}
	where
	\begin{align}
	    M_{\lambda^{k+1}} = -\nabla_\lambda L(\theta,\lambda)|_{\theta=\theta^*(\lambda^k),\lambda=\lambda^k} +\Bigl(\frac{1}{N}\sum_{j=1}^{N} \log\frac{\mathbb{P}_{\theta^*(\lambda^k)}(\tau^k_j)}{\mathbb{P}_\phi(\tau^k_j)}  -\delta \Bigr)\label{eq:mart-lambda}.
	\end{align}
	From \eqref{eq:mart-lambda}, we can verify that $\mathbb{E} \left[M_{\lambda^{k+1}}|\mathcal{F}_\lambda^k\right]= 0$, where $\mathcal{F}_\lambda^k = \sigma(\lambda^m,M_{\lambda^{m}}, m\leq k)$ is a filtration of $\lambda$ generated by different independent trajectories. Also, we have:
	\begin{align}
	    \|M_{\lambda^{k+1}}\|^2
	    \leq 2\| \nabla_\lambda L(\theta,\lambda)|_{\lambda 
	    = \lambda^k}\|^2 + \frac{2^N}{N}\left(\max_{1\leq j\leq N}\left|\log\frac{\mathbb{P}_{\theta^*(\lambda^k)}(\tau^k_j)}{\mathbb{P}_\theta(\tau^k_j)}-\delta\right|\right)^2<\infty.\nonumber
	\end{align}
	Hence, $M_{\lambda^{k+1}}$ is a Martingale difference error. Also, (\textit{v}) $\sup\{\lambda^k\}<\infty$. Recall that from the convergence analysis of the $\theta$-update for a $\lambda^k$, we know that $\theta^*(\lambda^k)$ is an asymptotically stable point. Then by Theorem 2 of Chapter 6 in \cite{borkar2009stochastic}, we can conclude that $(\theta^k,\lambda^k)$ converges almost surely to $(\theta^*(\lambda^*),\lambda^*)$, where $\lambda^*$ belongs to an internally chain transitive invariant set of \eqref{eq:update-lambda-cont}. 
	
	Define the Lyapunov function:
	\begin{align}
	    \mathcal{L}(\lambda) = - L(\theta^*(\lambda),\lambda) + L(\theta^*(\lambda^*),\lambda^*),\nonumber
	\end{align}
	where $\lambda^*$ is a local maximum point, i.e., there exists $r$ such that for any $\lambda\in \mathcal{B}_r(\lambda^*)$, the Lyapunov function $\mathcal{L}(\lambda)$ is positive definite. We can follow similar lines of arguments as we did for the $\theta$-update to show that $\frac{d\mathcal{L}(\lambda)}{dt}\leq 0$ and this quantity is non-zero as long as $\Psi_\Lambda(-\nabla_\lambda L(\theta^*(\lambda),\lambda))\neq 0$. Then by using the results of LaSalle's Local Invariant Set Theorem, we can establish the convergence of the dynamical system to the largest invariant set within $\mathbb{M}_\lambda\coloneqq \{\lambda: \Psi_\Lambda(-\nabla_\lambda L(\theta^*(\lambda),\lambda)) = 0\}$. This means that $\lambda^*\in \mathbb{M}_\lambda$ is a stationary point. Let $l = \min\{\mathcal{L}(\lambda):\Psi_\Lambda(-\nabla_\lambda L(\theta^*(\lambda),\lambda))=0 , \lambda\in \mathcal{B}_r(\lambda^*)\setminus \lambda^* \}$. Then, every trajectory starting with $\lambda^0$ in $\{\lambda\in\mathcal{B}_r(\lambda^*): \mathcal{L}(\lambda) < l \}$ will tend to $\lambda^*$ w.p.1. 
	
	\textbf{Saddle Point Analysis}: By denoting $\theta^*= \theta^*(\lambda^*)$, we want to show that $(\theta^*,\lambda^*)$ is, in fact, a saddle point of the Lagrangian $L(\theta,\lambda)$. Recall that, as we proved in the convergence the of $\theta$-update, $\theta^*$ is a local minimum of $L(\theta,\lambda)$ within a sufficiently small ball around itself, i.e., there exists $r>0$ such that
	\begin{align}
	    L(\theta^*,\lambda^*) \leq L(\theta,\lambda^*), \quad \forall \theta\in \Theta\cap \mathcal{B}_r(\theta^*).\label{eq:sad1}
	\end{align}
	It is easy to verify that $\theta^*$ is a feasible solution of \eqref{opt-r} whenever $\lambda^*\in [0,\lambda_{max})$, i.e.
	\begin{align}
	    D_{KL}(\theta^*\;\|\;\phi)\leq \delta.\label{eq:feasibility}
	\end{align}
	To show this, assume for a contradiction that $D_{KL}(\theta^*\;\|\;\phi) - \delta > 0$. Then,
	\begin{align}
	    \Psi_\Lambda\left[\nabla_\lambda L(\theta,\lambda)|_{\theta=\theta^*,\lambda=\lambda^*}\right] &
	   = \lim_{\alpha\downarrow 0}\frac{\Gamma_\Lambda\left[\lambda^* + \alpha \nabla_\lambda L(\theta,\lambda)|_{\theta=\theta^*,\lambda=\lambda^*}\right] -\Gamma_\Lambda\left[\lambda^*\right] }{\alpha}\nonumber\\
   	   &= \lim_{\alpha\downarrow 0}\frac{\Gamma_\Lambda\left[\lambda^* + \alpha \left(D_{KL}(\theta^*\;\|\;\phi) - \delta\right)\right] -\Gamma_\Lambda\left[\lambda^*\right] }{\alpha}\nonumber\\
	    &=  D_{KL}(\theta^*\;\|\;\phi) - \delta > 0,\nonumber
	\end{align}
	which contradicts the fact that   $\Psi_\Lambda\left[\nabla_\lambda L(\theta,\lambda)|_{\theta=\theta^*,\lambda=\lambda^*}\right]=0$. Notice that the feasibility cannot be verified when $\lambda^*=\lambda_{max}$, because $\Psi_\Lambda\left[\nabla_\lambda L(\theta,\lambda)|_{\theta=\theta^*(\lambda_{max}),\lambda=\lambda_{max}}\right]=0$ when $D_{KL}(\theta^*\;\|\;\phi)> \delta$. In this case, we increase $\lambda_{max}$ (e.g., we set $\lambda_{max}\leftarrow2\lambda_{max}$ in our algorithm) if such a behavior happens until it converges to an interior point of $[0,\lambda_{max}]$.  
	
	In addition, the complementary slackness condition 
	\begin{align}
	    \lambda^* (D_{KL}(\theta^*\;\|\;\phi)- \delta)=0 \label{eq:complementary}
	\end{align}
	holds. To show this, we only need to verify that $D_{KL}(\theta^*\;\|\;\phi)<\delta$ yields  $\lambda^*=0$. For a contradiction, suppose that $\lambda^*\in (0,\lambda_{max})$. Then, we have
	\begin{align}
	 \Psi_\Lambda\left[\nabla_\lambda L(\theta,\lambda)|_{\theta=\theta^*(\lambda^*),\lambda=\lambda^*}\right] &=
	   D_{KL}(\theta^*\;\|\;\phi) - \delta < 0,\nonumber
	\end{align}
	which contradicts the fact that $\Psi_\Lambda\left[\nabla_\lambda L(\theta,\lambda)|_{\theta=\theta^*,\lambda=\lambda^*}\right]=0$, meaning that $\lambda^*=0$ in this case. Hence, we have:
	\begin{align}
	    L(\theta^*,\lambda^*) &= V_{\theta^*}(x_0) + \lambda^* \left(D_{KL}(\theta^*\;\|\;\phi) - \delta\right)\nonumber\\
	    &= V_{\theta^*}(x_0) \nonumber\\
	    &\geq  V_{\theta^*}(x_0) + \lambda \left(D_{KL}(\theta^*\;\|\;\phi) - \delta\right) = L(\theta^*,\lambda).\label{eq:sad2}
	\end{align}
	
	From \eqref{eq:sad1} and \eqref{eq:sad2}, we observe that $(\theta^*,\lambda^*)$ is a saddle point of $L(\theta,\lambda)$, so according to the saddle point theorem, $\theta^*$ is a local minimum of \eqref{opt-r}. Recall that the result of Theorem \ref{thm:main} depends on the initial values for $\theta^0$ and $\lambda^0$, so the convergence to a local minimum is sample path depenedant.
	\end{proof}
	
    \subsection{Proof of Corollary \ref{cor-stationary}}\label{sec:app-coro-station}
    \begin{proof}
    
     From the convergence analysis of the $\theta$-update, we know that $\{\theta^k\}$ converges almost surely to the largest invariant set within $\mathbb{M}_\theta$, and similarly, $\{\lambda^k\}$ converges almost surely to the largest invariant set within $\mathbb{M}_\lambda$. We also know from \eqref{eq:feasibility} that $\theta^*$ is a feasible point of \eqref{opt-r}. When $\lambda^*=0$, then $ L(\theta^*,\lambda^*) =  V_{\theta^*}(x_0)$.  Also, for $\lambda^*>0$, the complementary slackness condition \eqref{eq:complementary} implies $D_{KL}(\theta^*\;\|\;\phi) = \delta$. Hence $\nabla_\theta D_{KL}(\theta\;\|\;\phi)|_{\theta=\theta^*} = 0$, which in turn, means that
    \begin{align}
        \nabla_\theta L(\theta,\lambda^*)|_{\theta=\theta^*} = \nabla_\theta V_\theta(x_0)|_{\theta=\theta^*} + \lambda^* \nabla_\theta D_{KL}(\theta\;\|\;\phi)|_{\theta=\theta^*} =  \nabla_\theta V_\theta(x_0)|_{\theta=\theta^*}.
    \end{align}
     Hence, for a $\theta^*$ located in the interior of $\Theta$, we have $\nabla_\theta L(\theta,\lambda^*)|_{\theta=\theta^*} = \nabla_\theta V_\theta(x_0)|_{\theta=\theta^*} = 0$, so it is a first-order stationary point of \eqref{opt-r}. However, if $\theta^*\in\partial \Theta$, it is possible to have $\|\nabla_\theta L(\theta,\lambda^*)|_{\theta=\theta^*}\|\neq 0$. 
     \end{proof}
     \begin{remark}
     In practice, we choose the projection set $\Theta$ large enough so that the latter case (convergence to boundary) will not happen. For example, assuring that the weights of a neural network do not diverge is a sufficient criterion to use instead of the projection operator $\Gamma_\Theta$.
     \end{remark}

     \subsection{Equivalent Results for \eqref{opt-f}}\label{sec:equivalence}
 	A similar PDPG algorithm to the one proposed in Algorithm \ref{alg:PG} can solve \eqref{opt-f},  only requiring a slight modification of rules \eqref{alg:eq:update-theta} and \eqref{alg:eq:update-lambda} as
	 \begin{align}
	 \theta^{k+1} &= \Gamma_{\Theta}\Bigl[\theta^{k} - \alpha_1(k)\Bigl( \frac{1}{N}\sum_{j=1}^{N} \nabla_{\theta}\log{\mathbb{P}_\theta(\tau^k_j)}\bigr|_{\theta=\theta^k}\bigl(J(\tau^k_j) + \lambda^k \;IS(\tau^k_j)\log\frac{\mathbb{P}_\phi(\tau^k_j)}{\mathbb{P}_\theta(\tau^k_j)}-\lambda^k \bigr)\Bigr)\Bigr]\nonumber\\
	 \lambda^{k+1} &= \Gamma_{\Lambda}\Bigl[\lambda^{k} + \alpha_2(k) \Bigl(\frac{1}{N}\sum_{j=1}^{N} \;IS(\tau^k_j) \log\frac{\mathbb{P}_\phi(\tau^k_j)}{\mathbb{P}_\theta(\tau^k_j)}  -\delta \Bigr)\Bigr],\nonumber
	 \end{align}
	 where $IS(\tau^k_j) ={\mathbb{P}_\phi(\tau^k_j)}/{\mathbb{P}_\theta(\tau^k_j)}$ is the importance sampling weight added to account for the bias introduced by sampling under the student's policy.
 	To ensure a well-defined \eqref{opt-f}, we need the following assumption:
    \begin{assumption}\label{assumption:feas-optf}
 	\textbf{Well-defined \eqref{opt-f}}: for any state--action pair $(x,a)\in\mathcal{X}\times\mathcal{A}$ with $\pi_S(x,a)=0$, we have $\pi_T(x,a)= 0$.
	 \end{assumption}
     This assumption ensures a similar criterion to that of  Assumption \ref{assumption:feas-optr}, but notice that in this case, the student might take any action, regardless of the teacher's policy. Exactly the same steps can be taken, virtually verbatim, to prove the following convergence property of the PDPG algorithm for \eqref{opt-f}.
	\begin{theorem}\label{thm:main-f}
	Under Assumptions \ref{assumption:differentiable}, \ref{assump:step}, and \ref{assumption:feas-optf}, the sequence of policy updates (starting from $\theta^0$ sufficiently close to a local optimum point $\theta^*$) and Lagrange multipliers converges almost surely to a saddle point of the Lagrangian, i.e., $(\theta(k),\lambda(k)) \overset{a.s.}{\longrightarrow} (\theta^* , \lambda^*$). Then, $\theta^*$ is the local optimal solution of \eqref{opt-f}.
	\end{theorem}
 	

	
	\section{Practical PDPG Algorithm}
	A naive implementation of Algorithm \ref{alg:PG} would result in a high-variance training procedure. In this section, we discuss several techniques for variance reduction, resulting in a more stable algorithm compared to the one proposed in Algorithm \ref{alg:PG}.

	\subsection{Step-wise KL-divergence Measure}\label{sec:measure}
	
	 In the policy distillation literature, some studies use a \textit{trajectory-wise} KL-divergence \eqref{eq:traj-wise-forward} as the distance metric \cite{teh2017distral}, but the \textit{step-wise} KL-divergence between the distribution is also common \cite{ghosh2017divide}, which is defined as:
	\begin{align}
	D^{step}_{KL}(\phi\;\|\;\theta)= \mathbb{E}_{x\sim d_{\pi_T}}\left[D_{KL}\bigl(\pi_T(\cdot|x;\phi)\;\|\;\pi_S(\cdot|x;\theta)\bigr)\right], \label{eq:kl-def}
\end{align}
where
\begin{align}
D_{KL}\bigl(\pi_T(\cdot|x;\phi)\;\|\;\pi_S(\cdot|x;\theta)\bigr)= \sum_{a\in\mathcal{A}}\pi_T(a|x;\phi) \log\frac{\pi_T(a|x;\phi)}{\pi_S(a|x;\theta)}.
\end{align}	 
	 In the next proposition, we explore the relations between these two methods.
	
	\begin{proposition}
		The following relation holds between the trajectory-wise and step-wise KL-divergence metrics:
		\begin{align}
		D_{KL}(\phi\;\|\;\theta) \leq \mathbb{E}[H] \;\;D^{step}_{KL}(\phi\;\|\;\theta)
		\end{align}
	\end{proposition}
	\begin{proof}
		According to the definition of trajectory-wise KL-divergence, we have:
		\begin{align}
		D_{KL}(\mathbb{P}_\phi (\tau) || \mathbb{P}_\theta (\tau) ) & =
		\sum_{\tau} \mathbb{P}_\phi(\tau) \log\frac{\mathbb{P}_\phi(\tau)}{\mathbb{P}_\theta(\tau)}\nonumber\\
		&=\sum_{\tau} \mathbb{P}_\phi(\tau) \log\frac{\mu(x_0)\prod_{t=0}^{H-1}\pi_T(a_t|x_t;\phi)P(x_{t+1}|x_t,a_t)}{\mu(x_0)\prod_{t=0}^{H-1}\pi_S(a_t|x_t;\theta)P(x_{t+1}|x_t,a_t)}\nonumber\\
		&=\sum_{\tau} \mathbb{P}_\phi(\tau) \sum_{t=0}^{H-1}\log\frac{\pi_T(a_t|x_t;\phi)}{\pi_S(a_t|x_t;\theta)}\nonumber\\
		&=\sum_{x\in\mathcal{X},a\in\mathcal{A}}\sum_{\tau} \mathbb{P}_\phi(\tau) \sum_{t=0}^{H-1} \mathbbm{I}_t (\tau;x,a) \log\frac{\pi_T(a|x;\phi)}{\pi_S(a|x;\theta)}\nonumber\\
		&\leq\sum_{x\in\mathcal{X},a\in\mathcal{A}}\mathbb{E}[H] d_{\pi_T}(x) \pi_T(a|x;\phi) \log\frac{\pi_T(a|x;\phi)}{\pi_S(a|x;\theta)}\nonumber\\
		&= \mathbb{E}[H] \sum_{x\in\mathcal{X}} d_{\pi_T}(x) \sum_{a\in\mathcal{A}}\pi_T(a|x;\phi) \log\frac{\pi_T(a|x;\phi)}{\pi_S(a|x;\theta)}\nonumber\\
		&=\mathbb{E}[H]\; \mathbb{E}_{x\sim d_{\pi_T}}\big[D_{KL}(\pi_T(\cdot|x;\phi) || \pi_S(\cdot|x;\theta) )\big]\nonumber
		\end{align}
		Here, $\mathbbm{I}_t (\tau;x,a)$ is the indicator of whether $(x_t=x,a_t=a)$ occurs along trajectory $\tau$. Also, $d_{\pi_T}(x)$ is the distribution of being in state $x$ under policy $\pi_T$, defined as
		\begin{align}
		d_{\pi_T}(x) = {\sum_{t=0}^{H_{max}} d_{t,\pi_T(x)}}/{H_{max}},\nonumber
		\end{align}
		and $d_{t,\pi_T(x)}$ is the probability of being in $x$ at time $t$ under policy $\pi_T$.
	\end{proof}
	According to this proposition, the step-wise KL distances can be used to provide an upper bound on the trajectory-wise one. In other words, if the step-wise KL multiplied by the expected horizon length are less than $\delta$, then it is also correct for the trajectory-wise one. 

	The only remaining issue is that computing the expectation in \eqref{eq:kl-def} is not straightforward, since we only have access to the sample trajectories of the student during training. Using student samples to approximate the KL-divergence introduces some bias. One can alleviate this bias by incorporating importance sampling (IS) weights as
	\begin{align}
	{D}^{step}_{KL}(\phi\;\|\;\theta)= \mathbb{E}_{x\sim d_{\pi_S}}\left[\frac{d_{\pi_T}(x)}{d_{\pi_S}(x)}{D}_{KL}\bigl(\pi_T(\cdot|x;\phi)\;\|\;\pi_S(\cdot|x;\theta)\bigr)\right]; 
	\end{align}
	however, computing the stationary distributions is still a challenging task, even in simple MDPs with finite state space. One can follow the instructions of \cite{liu2018breaking} for computing the correction values, but they add extra complications and are not the focus of this work. Even though we can more easily compute an (unbiased) estimate of reverse KL-divergence, we will utilize a biased estimation of the forward KL-divergence in most of our numerical analysis because of its ``mean-seeking'' property. Defining this biased forward KL-divergence is common in the literature, e.g., in \cite{schmitt2018kickstarting}.

    Next, we illustrate with an example the low variance of the step-wise approximators compared to the trajectory-wise one.
	
	\textbf{Example: KL Approximation Accuracy using Full Information }
	 We design a simple $2\times 2$ GridWorld example, as illustrated in Figure \ref{fig:2x2}, to visualize the effect of approximating \text{KL}-divergence using Monte Carlo sampling. There is one agent in the top-left corner of the grid and it should reach the goal state located in the bottom-left one. We kept the problem as simple as possible since we wanted to generate all possible trajectories for computing the exact KL-divergence. The length of the horizon for this game is 4, so the total number of possible trajectories is $4^4 = 256$. One may notice that some of these trajectories might fully overlap, but that is fine for the purpose of this experiment. In this experiment, we have used a linear function approximator (i.e., a neural network with no hidden layer) and a medium-sized neural network.
	\begin{figure}[htbp]
		\centering
		\includegraphics[width=0.15\linewidth]{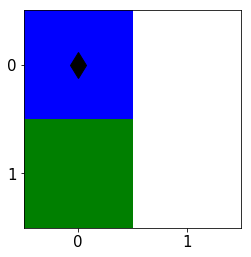}
		\caption{Illustration of the $2\times 2$ GridWorld used for evaluating the effectiveness of KL approximations}
		\label{fig:2x2}
	\end{figure}

 	We train a teacher that produces the actions right, left, up, and down with probabilities 0.7, 0.0, 0.1, and 0.2, respectively. Once the trained network is available, we initialize the student's policy variables with those of the teacher plus a random number. Figure \ref{fig:kl_exact} shows the convergence behavior of the KL approximations to the exact value as we increase the Monte Carlo samples. The horizontal axis shows the number of  sampled trajectories. As we observe,  step-wise KL can provide a very good approximation of KL, even with a single trajectory sample, but the trajectory-wise approximation exhibits unstable behavior which is due to the intrinsic high variance of the estimator.  

	\begin{figure*}[htbp]
	\centering
	\begin{subfigure}[t]{0.45\textwidth}
		\centering
		\includegraphics[width=\textwidth]{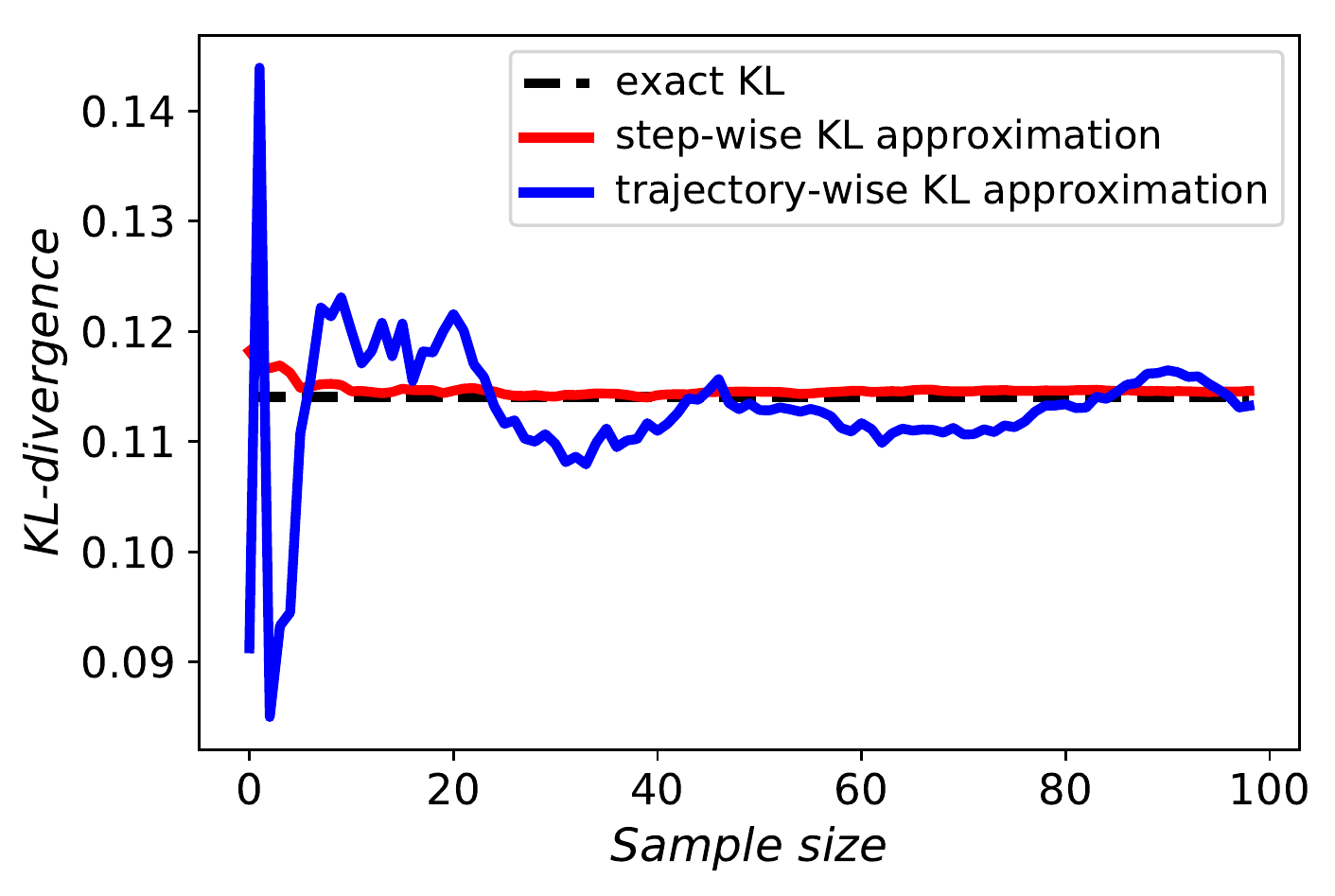}
		\caption{Linear policy approximation}
	\end{subfigure}%
	\begin{subfigure}[t]{0.45\textwidth}
		\centering
		\includegraphics[width=\textwidth]{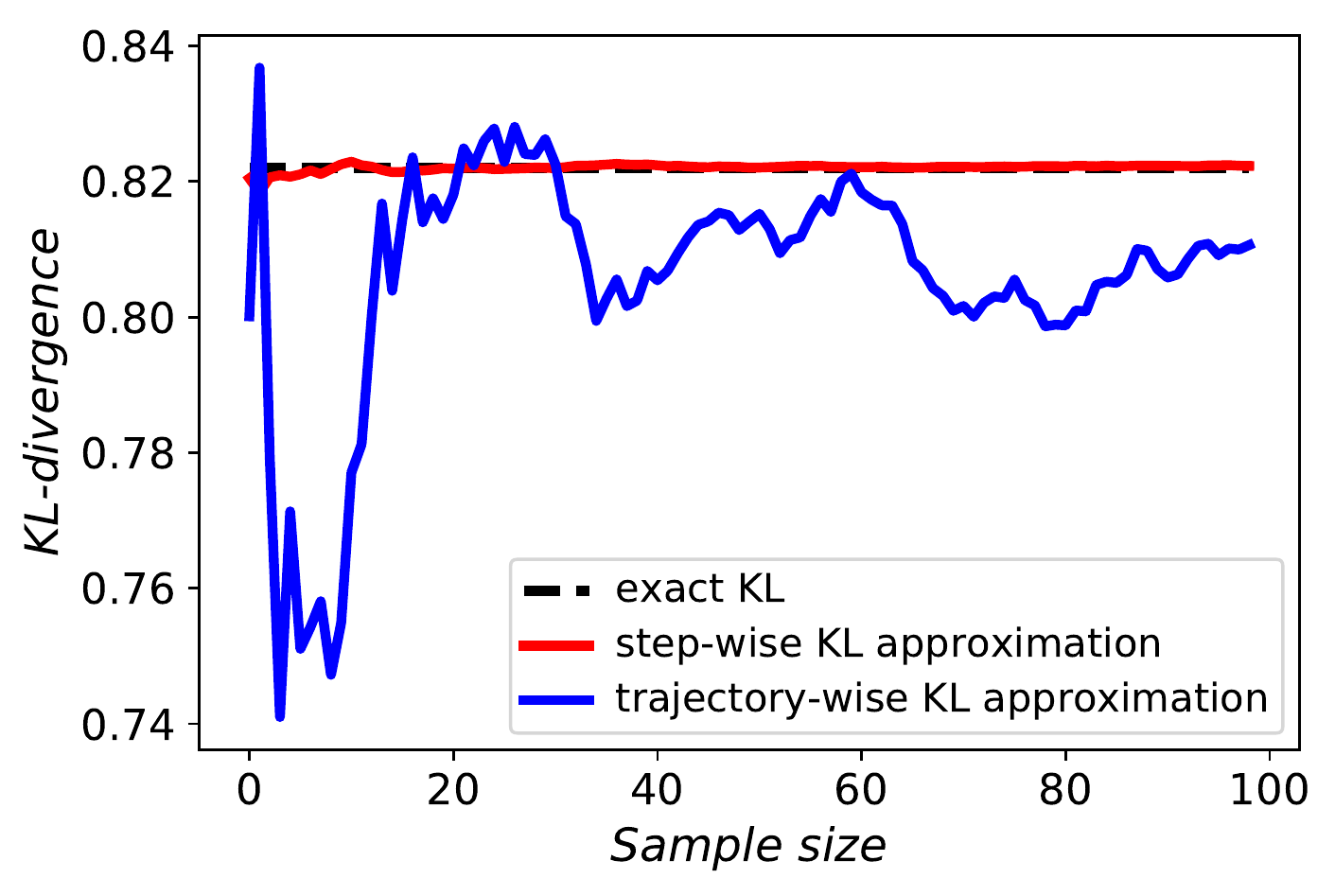}
		\caption{Neural network function approximator with variables}
	\end{subfigure}

	\caption{Comparison of step-wise and trajectory-wise KL approximations, and their convergence to the exact KLs for two different policy approximators}
	\label{fig:kl_exact}
\end{figure*}

    \subsection{Practical PDPG Algorithm}\label{app:ppg}
    According to the discussion of Section \ref{sec:prac-main}, we present the details of the practical PDPG algorithm in this section. We consider a setting in which the pre-trained teacher is readily available. The teacher articulates the status quo of solving the task. It can be a pre-trained RL agent itself, manually designed procedures, or a model of the teacher that has been trained using supervised learning from historical experiences. For example,  the square-wave experiment uses handcrafted tabular policies, while in the wall leaping, experiment the teacher's policy---modeled with a neural network---is the outcome of an actor--critic algorithm. As long as we have cheap access to the teacher throughout the algorithm for numerous queries and get the corresponding probabilities for any given state and action pair, it is sufficient for our purposes. 
    
   Our approach is described in Algorithm~\ref{alg:ppg}. 
    In every training iteration, we sample multiple trajectories under the student's policy, denoted by  $\mathcal{T}^k$, which will be further utilized in approximating the policy gradient, KL approximations, and entropy. For more sample efficiency of the algorithm, we extract multiple sub-trajectories from each $\tau^k_j$, and consider each sub-trajectory as an independent Monte Carlo sample. This is a common modification in policy gradient algorithms and  can provide a satisfactory approximation from a single trajectory experience. Then the teacher provides an approximate probability for all actions at all visited states $x^k_{j,t}$. Once we know the probability of both student and teacher, we can compute the approximate step-wise KL-divergence from steps \ref{algp:kl-approxi1} and \ref{algp:kl-approxi2}. Step \ref{algp:ent-approx} computes the entropy of student's current policy at each iteration. 
    	\begin{algorithm}[htbp]
    	\algsetup{linenosize=\tiny}
    	\footnotesize
		\caption{Practical Primal-Dual Policy Gradient (PDPG) Algorithm for \eqref{opt-f}}
		\label{alg:ppg}
		\begin{algorithmic}[1]
			\STATE \textbf{input:} teacher's policy with weights $\phi$
			\STATE \textbf{initialize:} student's policy with $\theta^0$, possibly equal $\phi$; initialize step size schedules $\alpha_1(\cdot)$, $\alpha_2(\cdot)$ and $\alpha_3(\cdot)$
			\WHILE {TRUE}
			\FOR {$k = 0,1,\cdots$}
			\STATE following policy $\theta^k$, generate a set of $N$ trajectories $\mathcal{T}^k=\{\tau_j^k,\,j =1,2,\cdots,N \}$, each starting from an initial state $x_0 \sim P_0(\cdot)$
			\STATE extract all trajectories $\tau_{j,t}^k$, which is a sub-trajectory of $\tau_{j}^k$ from $x_{j,t}^k$ onwards; also compute their corresponding accumulated reward $J(\tau^k_{j,t})$ and log-probability $\log{\tilde{\mathbb{P}}_\theta(\tau^{k}_{j,t})}\coloneqq \sum_{t=0}^{H_j^k-1}\log \pi_S(a_{j,t}^k|x_{j,t}^k,\theta)$. Let $\bar{\mathcal{T}}^k$ be the set of all sub-trajectories for all visited states $x^k_{j,t}$

			\STATE query the teacher and compute $\pi_T(\cdot|x_{j,t}^k,\theta^k)$
			\STATE \label{algp:kl-approxi1} compute KL-divergence for all visited states $x^k_{j,t}$, i.e.,
			 \begin{align}
            {D}^{step}_{KL}\bigl(\pi_T(\cdot|x^k_{j,t};\phi)\;\|\;\pi_S(\cdot|x^k_{j,t};\theta^k)\bigr) = \sum_{a\in\mathcal{A}}  \pi_T(a|x^k_{j,t};\phi) \log\frac{\pi_T(a|x^k_{j,t};\phi)}{\pi_S(a|x^k_{j,t};\theta^k)}, \quad\forall j,t
            \end{align}
	    
			\STATE \label{algp:kl-approxi2} \textbf{(\textit{KL approximation})} compute the approximate KL-divergence as
			\begin{align}
			    \hat{D}^{step}_{KL} (\phi\;\|\;\theta^k) = \frac{1}{N}\sum_{j=1}^{N} \frac{1}{H_j^k-1} \sum_{t=0}^{H_j^k-1} \textsc{clip}_{\rho}\left({D}^{step}_{KL}\bigl(\pi_T(\cdot|x^k_{j,t};\phi)\;\|\;\pi_S(\cdot|x^k_{j,t};\theta^k)\bigr)\right)
			\end{align}
			\STATE \label{algp:ent-approx}\textbf{(\textit{entropy approximation})} compute the approximate entropy 
		    \begin{align}
                e\hat{n}t(\theta^k) = - \frac{1}{N} \sum_{j=1}^N \frac{1}{H_j^k-1}\sum_{t=0}^{H_j^k-1} \sum_{a\in\mathcal{A}}  \pi_S(a|x^k_{j,t};\theta^k) \log{\pi_S(a|x^k_{j,t};\theta^k)}\label{algp:eq:ent-constraint}
            \end{align}
            \STATE \label{algp:loss}\textbf{(\textit{compute loss})} compute the loss according to
            \begin{align}
                Loss(\theta^k,\lambda^k,\zeta^k) = \frac{1}{|\bar{\mathcal{T}}^k|}\sum_{\tau_{j,t}^k \in \bar{\mathcal{T}}^k}&\log \tilde{\mathbb{P}}(\tau_{j,t}^k)(J(\tau_{j,t}^k)-V(x_{j,t}^k)) +\nonumber\\
                \lambda^k &(\hat{D}^{step}_{KL} (\phi\;\|\;\theta^k) - \delta)+ \zeta^k (e\hat{n}t(\theta^k) - \delta^{ent})
            \end{align}

			\STATE \label{algp:pg:theta-step} \textbf{($\theta$\textit{-update})} update $\theta^k$ according to  
			\begin{align}
			\theta^{k+1} = \Gamma_{\Theta}\Bigl[\theta^{k} - \alpha_1(k)\Bigl( \frac{1}{N}\sum_{j=1}^{N} \nabla_{\theta}Loss(\theta,\lambda^k,\zeta^k)\bigr|_{\theta=\theta^k}\Bigr)\Bigr]\label{algp:eq:update-theta}
			\end{align}
			\STATE \label{algp:eq:update-lambda} \textbf{($\lambda$\textit{-update})} update $\lambda^k$ according to
			\begin{align}
			\lambda^{k+1} = \Gamma_{\Lambda}\Bigl[\lambda^{k} + \alpha_2(k) \Bigl(\hat{D}^{step}_{KL} (\phi\;\|\;\theta^k) -\delta \Bigr)\Bigr]
			\end{align}
			\STATE \label{algp:eq:update-zeta} \textbf{($\zeta$\textit{-update})} update $\zeta^k$ with rule
			\begin{align}
			\lambda^{k+1} = \Gamma_{Z}\Bigl[\lambda^{k} + \alpha_3(k) \Bigl(e\hat{n}t(\theta^k) -\delta^{ent} \Bigr)\Bigr]
			\end{align}
			\ENDFOR
			\STATE update $\lambda_{max}$ similar to Algorithm \ref{alg:PG}
			\IF {$\zeta^k$ converges to $\zeta_{max}$} \label{algp:pg:lambdamaxbeg}
			\STATE $\zeta_{max} \leftarrow  \zeta_{max} + constant$
			\ELSIF  {$\zeta^k$ converges to $\zeta_{min}$}
			\STATE $\zeta_{min} \leftarrow  \zeta_{min} - constant$
			\ELSE {}
			\STATE return $\theta$, $\lambda$ and $\zeta$; break
			
			\ENDIF\label{algp:pg:lambdamaxend}
			
			\ENDWHILE
		\end{algorithmic}
	\end{algorithm}

	
	Now, we have all approximations for computing update directions. In step \ref{algp:loss}, we use all previously computed sub-trajectory log-probabilities and their cumulative sampled reward along with the KL and entropy approximation to compute the loss. In this step, we also use a critic to provide a value of being at the initial point of each sub-trajectory $V(x_{j,t}^k)$, which will provide a baseline for variance reduction. Note that we didn't include the critic steps in our main algorithm since it follows a standard actor--critic design. Step \ref{algp:pg:theta-step} updates the policy parameters using the approximate gradient of loss with respect to $\theta$ at point $\theta^k$. To be precise, the approximate gradient is employed in a first-order optimizer, e.g., ADAM \cite{kingma2014adam}, to update the $\theta$ values in the descent direction of the loss. Finally, the Lagrange multipliers $\lambda$ and $\zeta$ are updated based on the amount of constraint violation at steps \ref{algp:eq:update-lambda} and \ref{algp:eq:update-zeta}. We also periodically check to see whether $\lambda^k$ has converged to $\lambda_{max}$, in which case we increase its quantity similar to Algorithm \ref{alg:PG}. Notice that since we have considered an equality constraint for entropy, its Lagrange multipliers can be positive or negative. To this end, we consider $\zeta\in[\zeta_{min},\zeta_{max}]$ and if it converges to the boundary, we will increase the interval length.


	\section{Extended Details on Experiments}
	In this section, we describe our environmental setup and hyper-parameters used. We also discuss more extensive results from the square-wave experiment.
	\subsection{Problem Setup}
	In all of our experiments, the first step was to identify the teacher. In the square-wave experiment, we manually designed all teacher probabilities at every state. We also modeled situations  in which the teacher is less ``determined'' and follows a more complicated decision-making scheme, such as in our wall leaping experiment, in which the teacher is the policy of an agent trained using the actor--critic algorithm.

 Although we could have initialized the student's policy randomly, we chose to initialize it with a pre-trained neural network. In all experiments, we train the neural network for the unconstrained problem and using the actor--critic algorithm.  Similarly, the student's critic is initialized from the previously trained critic. Notice that the student's initial policy does not need to be the same as the teacher's policy. For example, the initial student policy in the square-wave experiment always takes the horizontal path, which is totally different from that of the teacher. Nevertheless, starting from a policy close to the teacher would expedite the learning process since there is a high probability of finding an improved policy in the proximity of the teacher.
    
    However, starting from a previously trained network can bring some difficulties. For example, having a deterministic initial policy would lead to a limited amount of exploration. To mitigate this issue, we use a temperature hyper-parameter when sampling from our $\softmax$, similar to \cite{bello2016neural}. In this method, we normalize the output of the neural network---called \textit{logits}---with temperature and then compute the sampling probabilities as $\pi(\cdot| s;\theta) = \softmax(\;logits\;/\;temperature\;)$. Using a temperature greater than 1 smoothes out the sampling probability distribution, so there will be a higher chance of visiting less-explored states.
    
    In all of our experiments, we used a neural network with two hidden layers, each with 64 neurons. We used the ADAM optimizer \cite{kingma2014adam} with step size $1\mathrm{e}{-}3$ to update the student's policy and critic. The temperature is 5; $\lambda$ and $\zeta$ start from 1. The learning rate for $\lambda$ and $\zeta$ starts from $1\mathrm{e}{-}3$ and decays to $1\mathrm{e}{-}3$ during training. The right hand sides of all entropy constraints are set to $0.02$. We also have a plan to open-source our PyTorch code upon publication.

    \subsection{Extended Results from Square Wave Experiment}
    This section includes extended results for the square-wave experiment. 
    
    \subsubsection{Effect of $\delta$ on Lagrange Multipliers }
    
    In this part, we illustrate the convergence of both Lagrange multipliers $\lambda$ and $\zeta$ to some steady values. From duality theory, we know that the converged values are a function of $\delta$, and in Figure \ref{fig:rel_lag}, we delineate these quantities for four different $\delta$. In Figure \ref{fig:rel_lambda}, we observe that as we increase $\lambda$ (i.e. relax the constraint), we will converge to a smaller $\lambda^*$. A similar monotonic relation is observed in Figure \ref{fig:rel_zeta}. The latter observation hypothesizes that in the square-wave experiment, larger $\delta$ values would bring more stochasticity to the optimal policy.
    
	  \begin{figure*}[htbp]
		\centering
		\begin{subfigure}[t]{0.48\columnwidth}
			\centering
			\includegraphics[width=\columnwidth,trim=.1cm .1cm .1cm .1cm,clip]{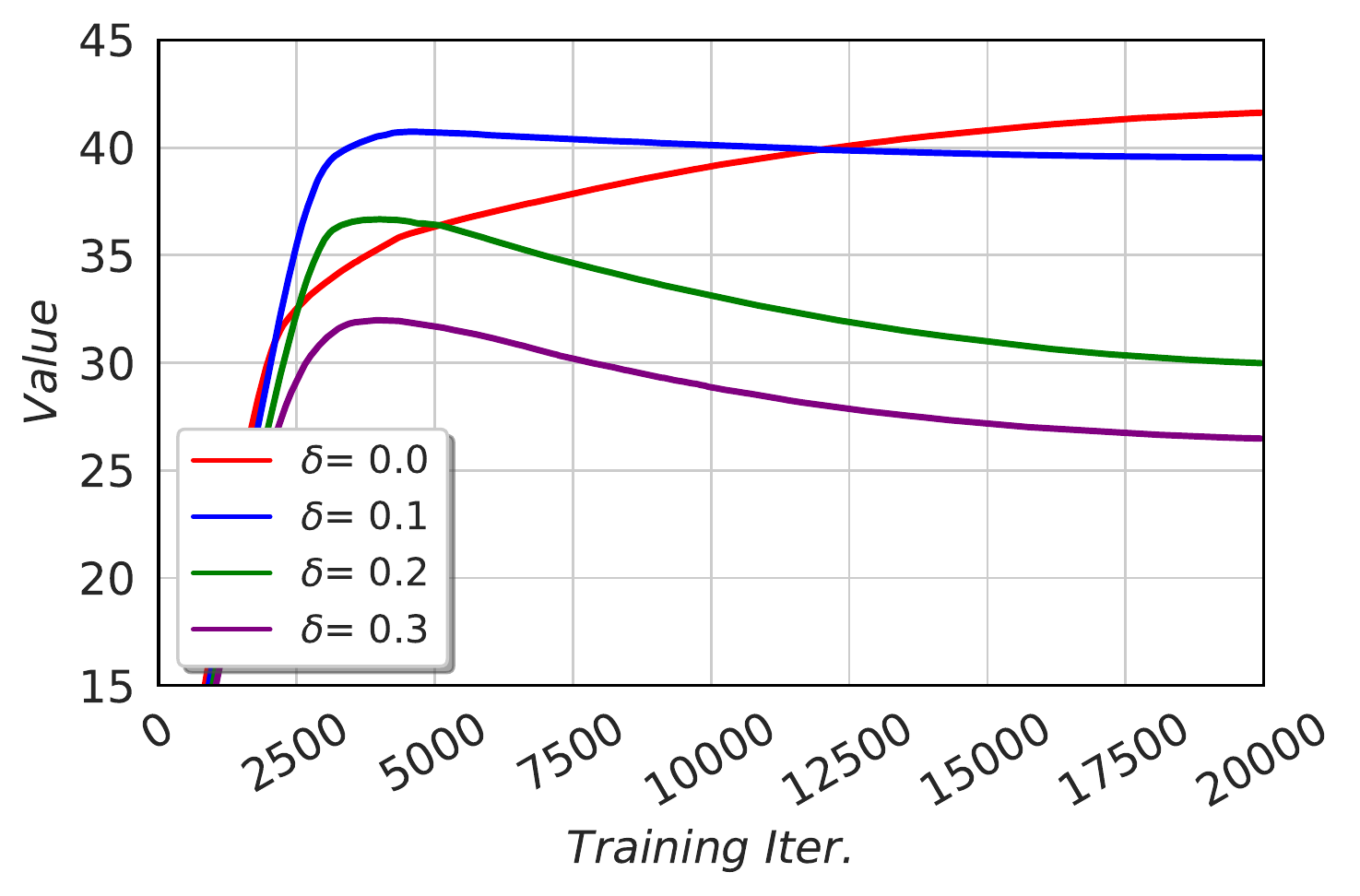}
			\caption{$\lambda$ convergence}
			\label{fig:rel_lambda}
		\end{subfigure}
		\hspace{.1cm}
		\begin{subfigure}[t]{0.48\columnwidth}
			\centering
			\includegraphics[width=\columnwidth,trim=.1cm .1cm .1cm .1cm,clip]{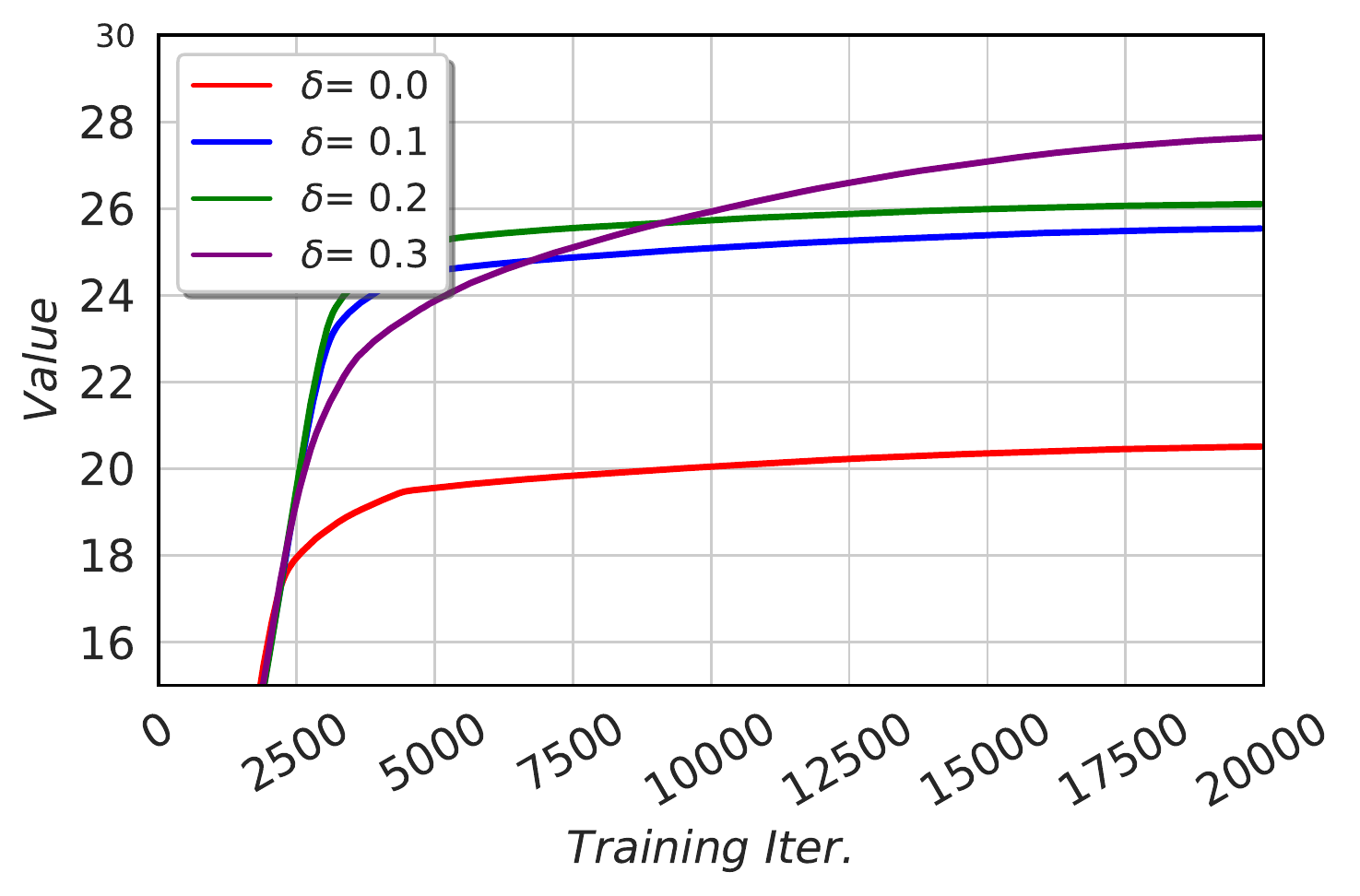}
			\caption{$\zeta$ convergence}
			\label{fig:rel_zeta}
		\end{subfigure}
		\caption{The effect of $\delta$ on Lagrange multipliers}
		\label{fig:rel_lag}
	\end{figure*}
	
    \subsubsection{Effect of Using Reverse KL-divergence}
    Note that throughout the experiments up to this point, we have used the forward KL-divergence. In this experiment, we intend to use the reverse KL constraint instead of the forward one to see how it affects learning. As we observe in Figure \ref{fig:full_rew-delta-rev}, the student  always converges to the teacher, no matter what the value of $\delta$ is. This is consistent with our theory that the student will converge to a sub-policy of the teacher. In fact, since the square-wave path is the only way that the teacher can reach the target, it will also be the optimal path for the student as well. Figure \ref{fig:full_rew-rho-rev} shows that adding  KL-clipping leads to different performance levels. 
    
    In Figure \ref{fig:part_env_zig_res_rev}, we study the effect of using the reverse KL-constraint in learning from the less confident teacher. As we observe from Figure \ref{fig:part_rew-delta-rev}, the student has converged to a policy similar to the square-wave policy, but with a few random actions. As we increase $\delta$, we allow the student to disagree with the teacher more in the less confident states, so her policy becomes closer to the square-wave path (with reward 60). Also, Figure \ref{fig:part_rew-rho-rev} shows a similar result as before on how KL-clipping may increase the reward.
    
	  \begin{figure*}[htbp]
		\centering
		\begin{subfigure}[t]{0.48\columnwidth}
			\centering
			\includegraphics[width=\columnwidth,trim=.1cm .1cm .1cm .1cm,clip]{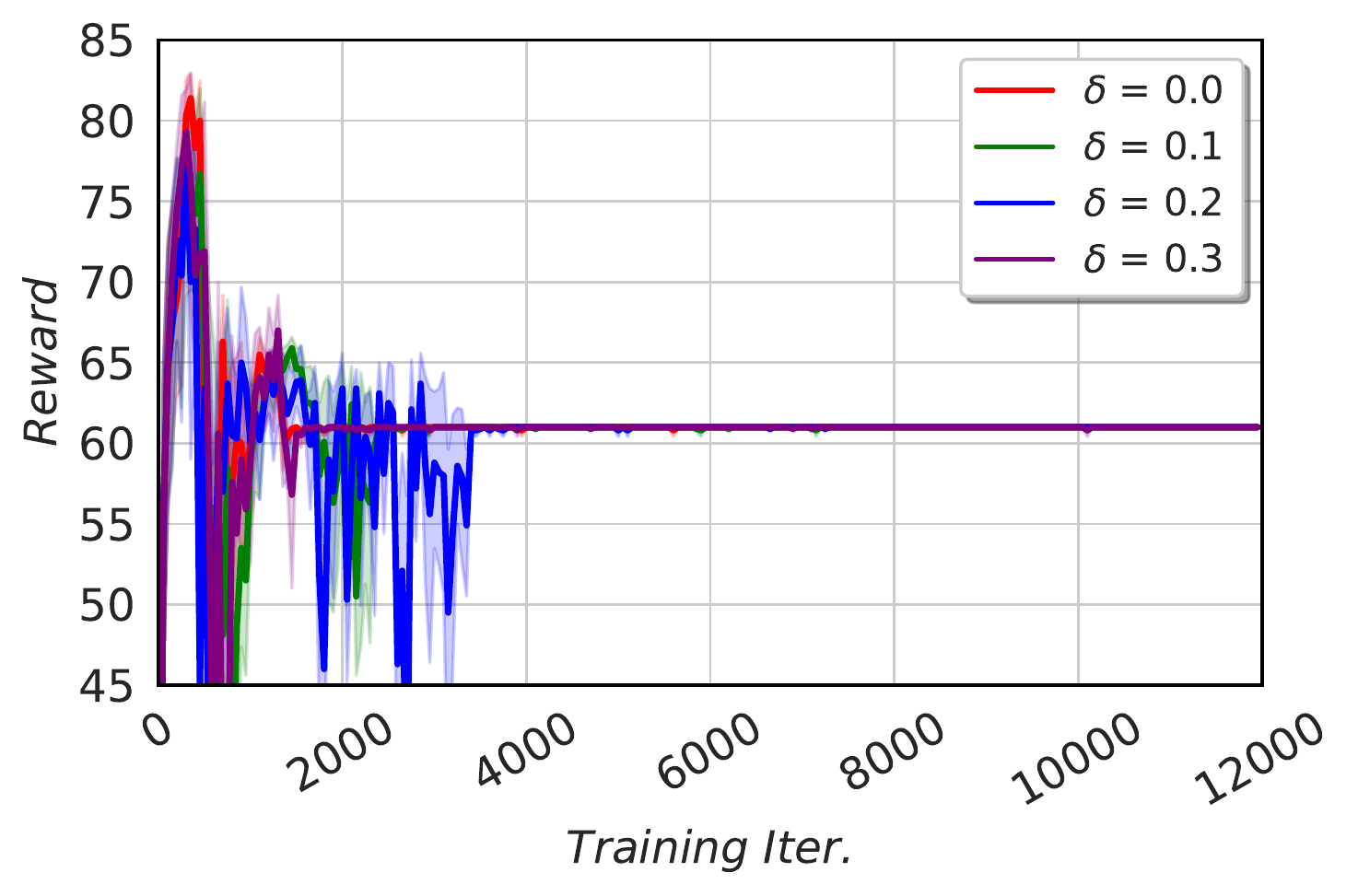}
			\caption{The effect of $\delta$ on reward; no KL-clipping}
			\label{fig:full_rew-delta-rev}
		\end{subfigure}
		\hspace{.1cm}
		\begin{subfigure}[t]{0.48\columnwidth}
			\centering
			\includegraphics[width=\columnwidth,trim=.1cm .1cm .1cm .1cm,clip]{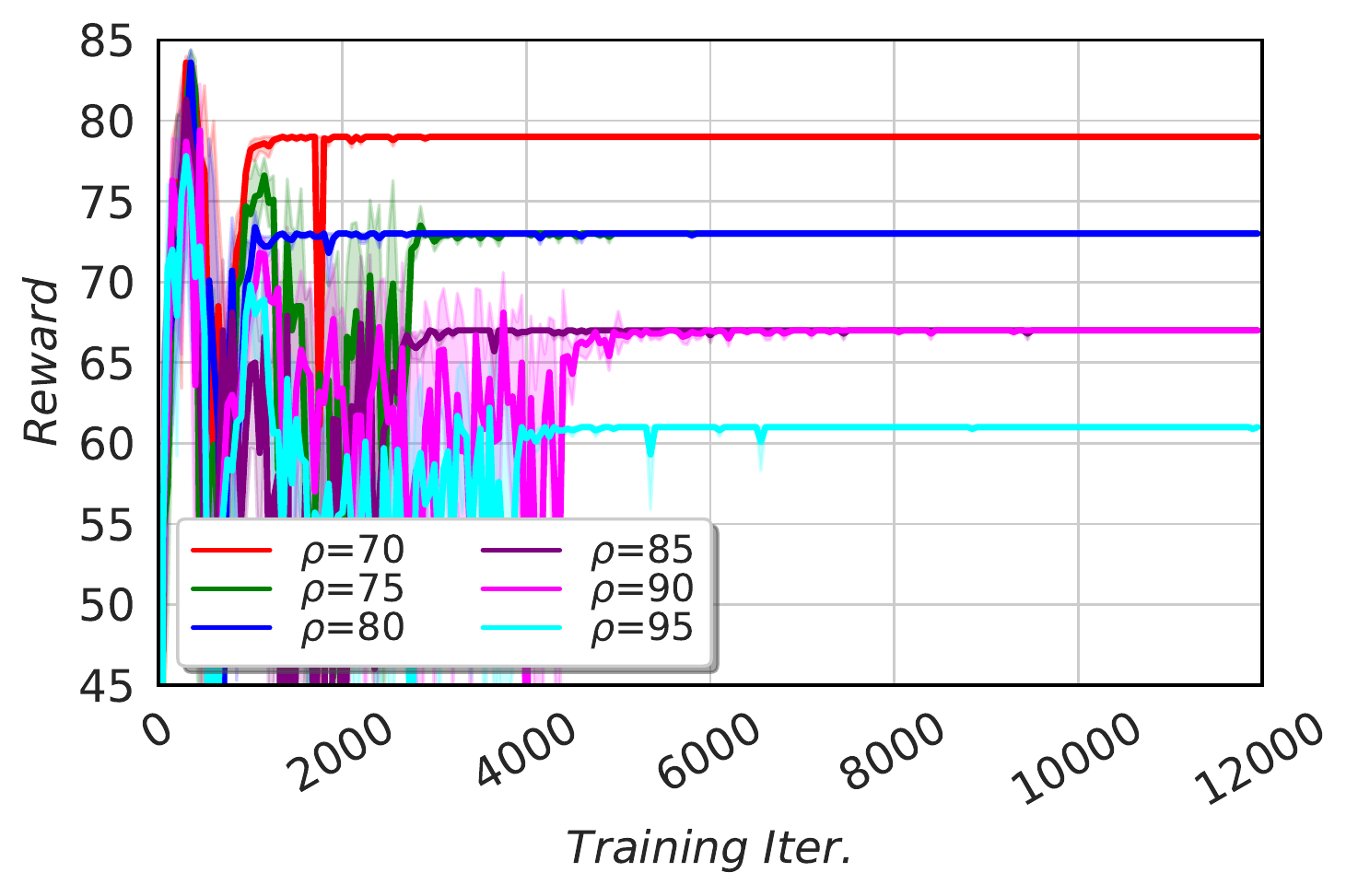}
			\caption{Total reward for different $\rho$ and $\delta=0.2$}
			\label{fig:full_rew-rho-rev}
		\end{subfigure}
		\caption{Performance of a student learning from the determined teacher using the reverse KL constraint}
		\label{fig:full_env_zig_res_rev}
	\end{figure*}
	
		  \begin{figure*}[htbp]
		\centering
		\begin{subfigure}[t]{0.48\columnwidth}
			\centering
			\includegraphics[width=\columnwidth,trim=.1cm .1cm .1cm .1cm,clip]{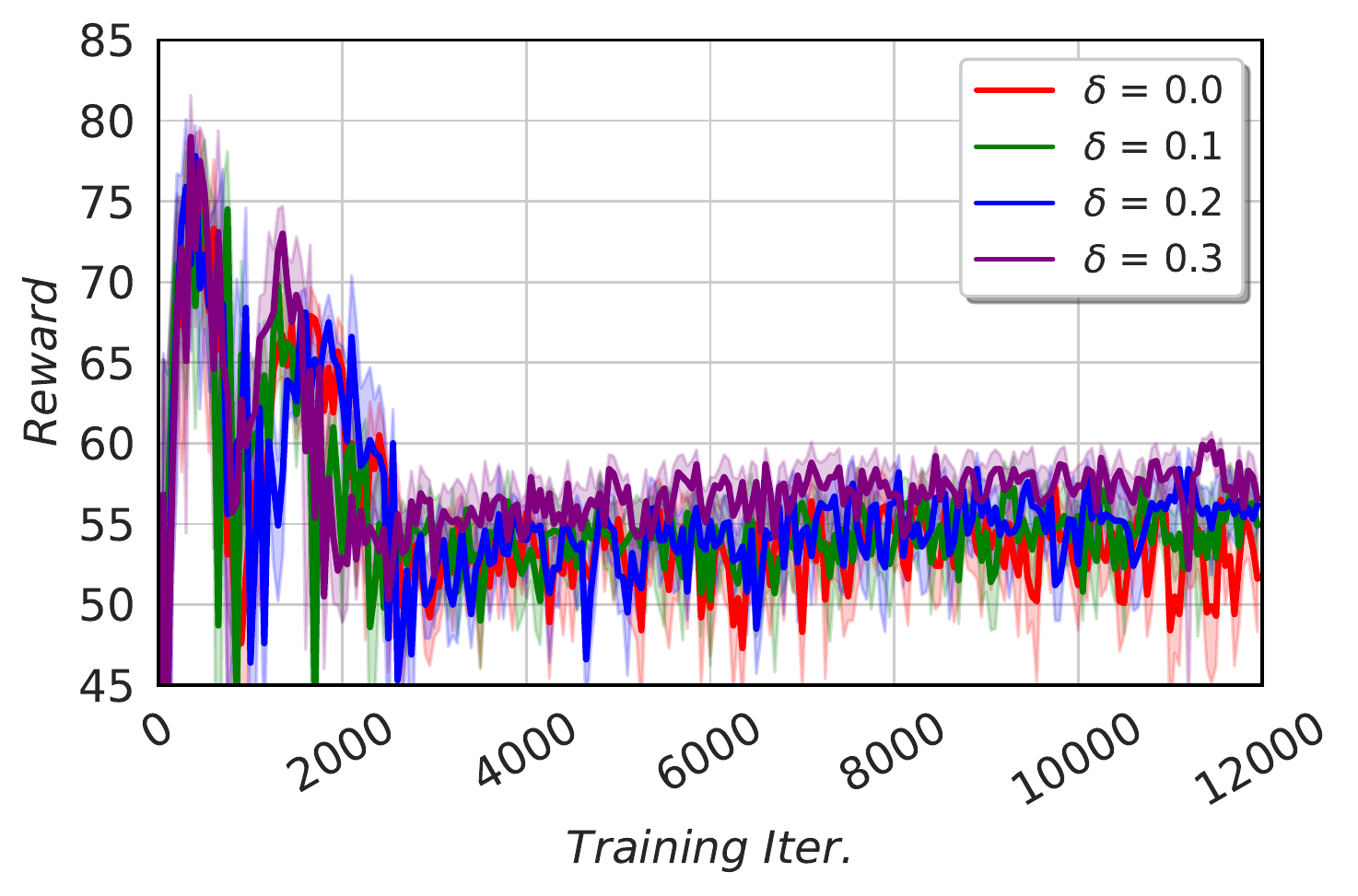}
			\caption{The effect of $\delta$ on reward; no KL-clipping}
			\label{fig:part_rew-delta-rev}
		\end{subfigure}
		\hspace{.1cm}
		\begin{subfigure}[t]{0.48\columnwidth}
			\centering
			\includegraphics[width=\columnwidth,trim=.1cm .1cm .1cm .1cm,clip]{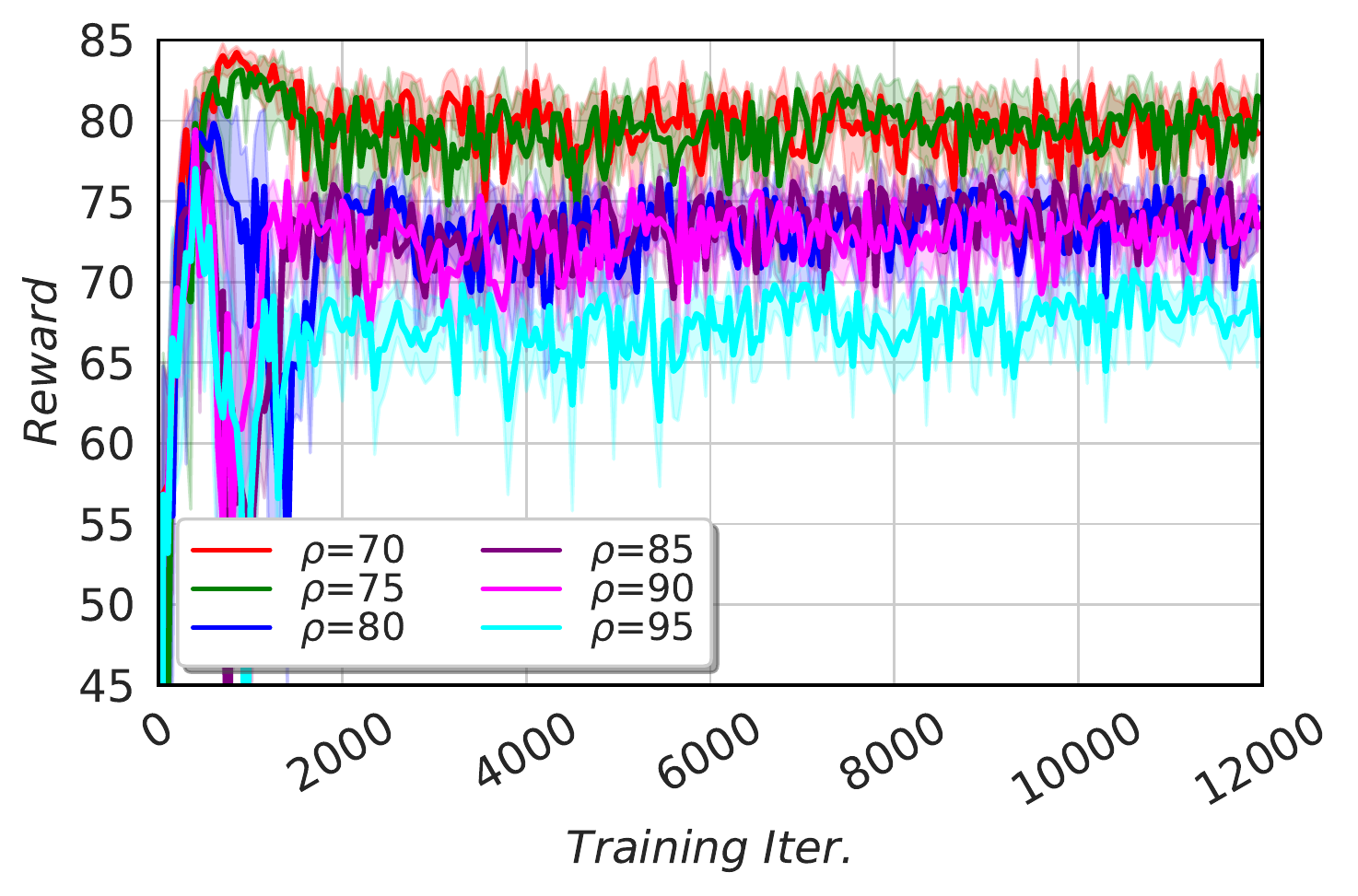}
			\caption{Total reward for different $\rho$ and $\delta=0.2$}
			\label{fig:part_rew-rho-rev}
		\end{subfigure}
		\caption{Performance of a student learning from the less confident teacher using the reverse KL constraint}
		\label{fig:part_env_zig_res_rev}
	\end{figure*}
	
	\subsubsection{Hellinger Constraint}
	
	As we discussed in Section \ref{sec:prac-main}, one way to reduce the stochasticity of the converged policy might be using a finite distance measure such as the Hellinger metric in our constraints. Figure \ref{fig:full_rew-H} shows the reward attained for the cases with and without the entropy constraint. As we see, the student has converged to a policy with reward 85, which corresponds to the horizontal path.  Hence, she was not successful in learning from the teacher. We tried different configurations as well, but all exhibit similar behavior. Figure \ref{fig:full_ent-H} shows that without using the entropy constraint, the student has relatively high entropy, but she is still unable to follow the teacher.

	\begin{figure*}[htbp]
		\centering
		\begin{subfigure}[t]{0.48\columnwidth}
			\centering
			\includegraphics[width=\columnwidth,trim=.1cm .1cm .1cm .1cm,clip]{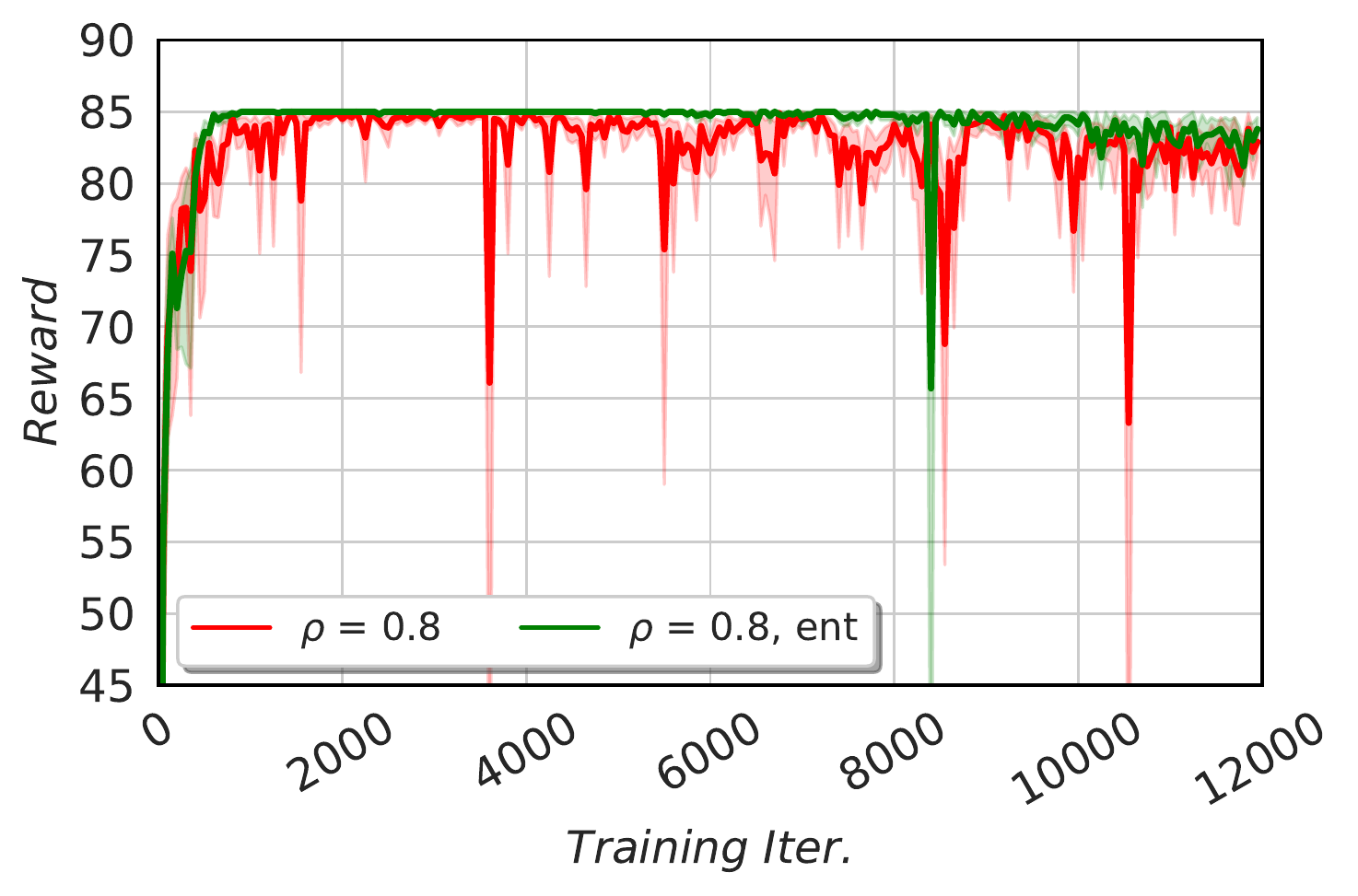}
			\caption{Reward gained with Hellinger constraint for $\delta=0.2$}
			\label{fig:full_rew-H}
		\end{subfigure}
		\hspace{.1cm}
		\begin{subfigure}[t]{0.48\columnwidth}
			\centering
			\includegraphics[width=\columnwidth,trim=.1cm .1cm .1cm .1cm,clip]{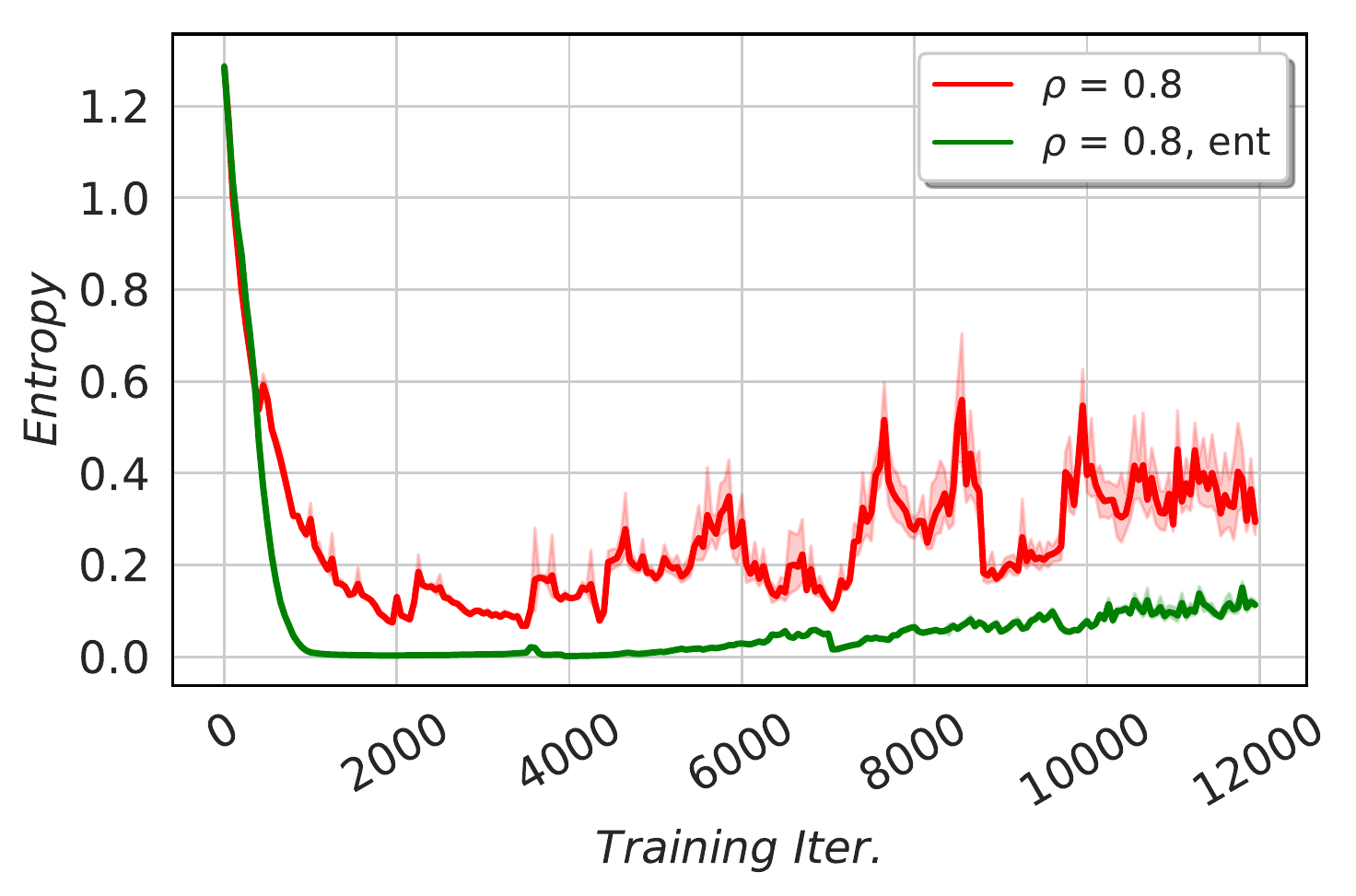}
			\caption{Total reward for different $\rho$ and $\delta=0.2$}
			\label{fig:full_ent-H}
		\end{subfigure}
		\caption{Performance of a student learning from determined teacher using the Hellinger constraint}
		\label{fig:full_env_zig_H}
	\end{figure*}

\end{document}